%% file: main.tex
\documentclass[10pt,twocolumn,letterpaper]{article}

\usepackage{wacv}              
\usepackage{multirow}
\usepackage[table]{xcolor}
\definecolor{Gray}{RGB}{255, 200, 220}
\usepackage{amsthm}
\newtheorem{theorem}{Theorem}[section]
\newtheorem{lemma}[theorem]{Lemma}

\definecolor{wacvblue}{rgb}{0.21,0.49,0.74}
\usepackage[pagebackref,breaklinks,colorlinks,allcolors=wacvblue]{hyperref}

\def\wacvPaperID{1208} 
\def\confName{WACV}
\def\confYear{2026}

\title{Countering Multi-modal Representation Collapse through Rank-targeted Fusion}

\author{%
  Seulgi Kim, Kiran Kokilepersaud, Mohit Prabhushankar, Ghassan AlRegib \\
  Georgia Institute of Technology, Atlanta, GA, USA \\
  \texttt{\{seulgi.kim, kpk6, mohit.p, alregib\}@gatech.edu} \\
\thanks{This work is supported by the ML4Seismic Consortium at Georgia Tech.}}

\begin{document}

\twocolumn[{%

{ \large
\begin{itemize}[leftmargin=2.5cm, align=parleft, labelsep=2cm, itemsep=4ex,]

\item[\textbf{Citation}]{S. Kim, K. Kokilepersaud, M. Prabhushankar, G. AlRegib, ``Countering Multi-modal Representation Collapse through Rank-targeted Fusion," in \textit{2026 IEEE/CVF Winter Conference on Applications of Computer Vision (WACV), Tucson, Arizona, USA, 2026.}}

\item[\textbf{Review}]{Date of Acceptance: Nov. 9th 2025}

\item[\textbf{Codes}]{\url{https://github.com/olivesgatech/R3D}}

\item[\textbf{Bib}]  {@inproceedings\{kim2025countering,\\
    title=\{Countering Multi-modal Representation Collapse through Rank-targeted Fusion\},\\
    author=\{Kim, Seulgi and Kokilepersaud, Kiran and Prabhushankar, Mohit and AlRegib, Ghassan\},\\
    booktitle=\{IEEE/CVF Winter Conference on Applications of Computer Vision\},\\
    year=\{2026\}\}}


\item[\textbf{Contact}]{
\{seulgi.kim, kpk6, mohit.p, alregib\}@gatech.edu\\\url{https://alregib.ece.gatech.edu}\\}

\item[\textbf{Main\\contact}]{
alregib@gatech.edu}
\end{itemize}

}}]
\newpage

\maketitle
\input{sec/0_abstract}    
\input{sec/1_intro_kiran_sample}
\input{sec/2_related_works}
\input{sec/3_analysis}
\input{sec/4_methods}
\input{sec/5_experiments}

\input{sec/6_conclusion}
{
    \small
    \bibliographystyle{ieeenat_fullname}
    \bibliography{main}
}
\clearpage
\input{sec/7_supplementary}

\end{document}

%% file: sec/0_abstract.tex
\begin{abstract}
Multi-modal fusion methods often suffer from two types of representation collapse: feature collapse where individual dimensions lose their discriminative power (as measured by eigenspectra), and modality collapse where one dominant modality overwhelms the other. Applications like human action anticipation that require fusing multifarious sensor data are hindered by both feature and modality collapse. However, existing methods attempt to counter feature collapse and modality collapse separately. This is because there is no unifying framework that efficiently addresses feature and modality collapse in conjunction. In this paper, we posit the utility of effective rank as an informative measure that can be utilized to quantify and counter both the representation collapses. We propose \textit{Rank-enhancing Token Fuser}, a theoretically grounded fusion framework that selectively blends less informative features from one modality with complementary features from another modality. We show that our method increases the effective rank of the fused representation. To address modality collapse, we evaluate modality combinations that mutually increase each others' effective rank. We show that depth maintains representational balance when fused with RGB, avoiding modality collapse. We validate our method on action anticipation, where we present \texttt{R3D}, a depth-informed fusion framework. Extensive experiments on NTURGBD, UTKinect, and DARai demonstrate that our approach significantly outperforms prior state-of-the-art methods by up to 3.74\%. Our code is available at: \href{https://github.com/olivesgatech/R3D}{https://github.com/olivesgatech/R3D}.

\end{abstract}


%% file: sec/1_intro_kiran_sample.tex
\begin{figure}[t!]
\begin{center}
\includegraphics[width=\linewidth]{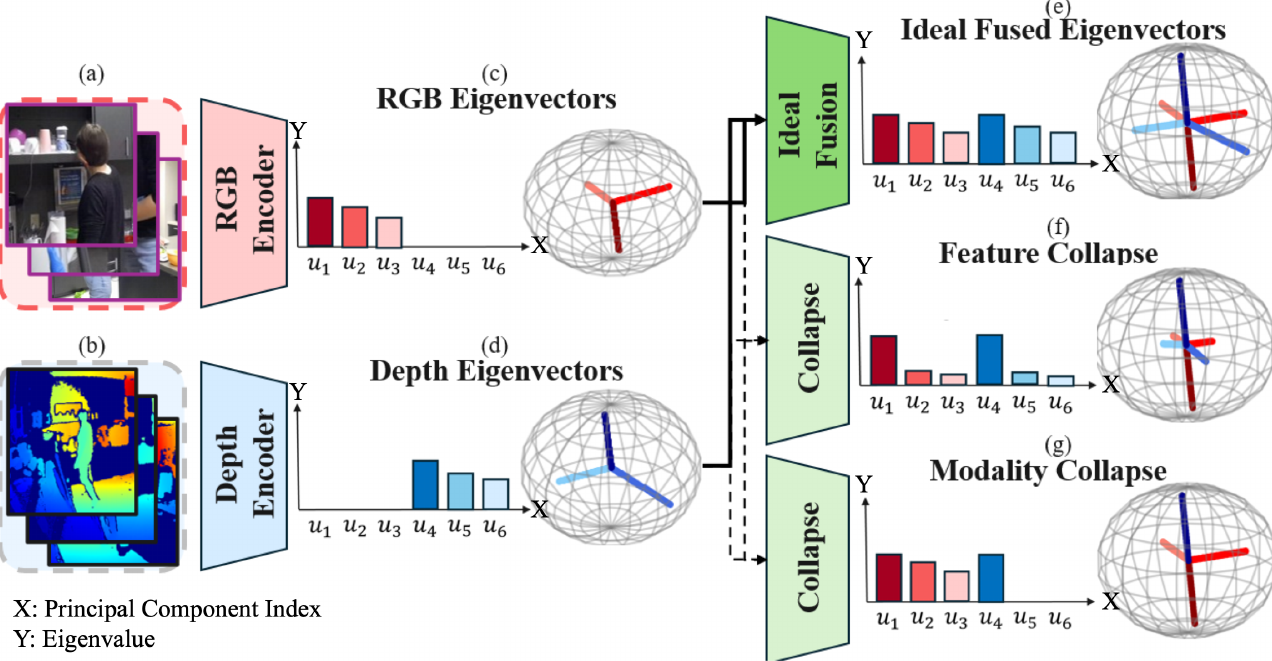}
\end{center}
\vspace{-0.7cm}
   \caption{This is a toy figure describing feature and modality collapse using spectral decomposition. (e): The ideal fused representation preserves complementary eigenvectors from both data modalities. (f): In contrast, feature collapse occurs when the fused representation varies along a subset of eigenvectors. (g): Modality collapse occurs when one modality dominates and suppresses the contribution of the eigenvectors of the other modality.
}
\label{fig:figure1}
\vspace{-0.5cm}
\end{figure}

\section{Introduction}
\label{sec:intro}
Modern data collection paradigms \cite{DARaiDataset} utilize multifarious sensor modalities to represent an environment, as each modality captures unique aspects of a scene. For example, RGB excels at conveying visual cues in objects such as color or texture, while depth captures geometric structures and directional relationships between objects (Figure~\ref{fig:figure1} (a-b)). Hence, multi-modal learning \cite{yuan2025survey} is indispensable for human-centric tasks such as action anticipation~\citep{kim2025multi,zhong2023anticipative,barsoum20203d, guan2022afe, cao2022qmednet, pmlr-v28-koppula13}, video understanding~\citep{huang2024multi, kaviani2025hierarchical}, and action recognition~\citep{yarici2025subject}, where single modalities fall short in expressing subtle semantic variations within or across scenes. Recent multi-modal fusion methods have achieved empirical success through contrastive losses \cite{radford2021learning,girdhar2023imagebind}, attention blocks \cite{yu2019deep}, or various specialized objectives \cite{fu2021seamless}. However, they often rely on learning representations via indirect pre-text tasks or alignment losses~\citep{liu2025tacfn}, rather than directly targeting the informative content in each modality. This lack of \textit{targeted information fusion} makes multi-modal fusion models susceptible to the common pitfalls: feature collapse ~\citep{laurent2023feature} and modality collapse \cite{chaudhuri2025closer}. A central technical challenge we tackle in this paper is the lack of a unifying information framework that targets both feature and modality collapse.

We illustrate \textit{targeted information fusion} in Figure \ref{fig:figure1}. In (c) and (d), RGB and Depth are projected into their respective representation spaces using their corresponding encoders. Each space is visualized with a bar plot of eigenvalues associated with eigenvectors $[u_1, ..., u_6]$ and a toy 3D spherical projection of the data points. Intuitively, the eigenvectors of each space represent the principal directions of variation in modality-specific representations. For the purposes of discussion, we define each eigenvector as a combination of projections from individual features (channels). Under this view, a channel is considered \textit{informative} if it contributes significantly to key eigenvector directions, indicating strong alignment with the principal variations in the space. Thus, we interpret eigenvectors in Figure~\ref{fig:figure1} as \textit{information components}, that is, the feature directions encoding useful content from each modality.

An ideal representation space after fusion should reflect Figure~\ref{fig:figure1}(e), where each information component is equally represented, resulting in a space that integrates feature directions encoding useful content from both modalities. However, in reality, different types of collapse can occur in the fused representation space. In Figure~\ref{fig:figure1}(f), we visualize feature collapse where certain information components become limited in their contribution to the overall space after fusion. This limits the overall feature diversity and the generalizability of the model \cite{kokilepersaud2025adadim, benkert2024transitional}. In Figure~\ref{fig:figure1}(g), we observe modality collapse \cite{chaudhuri2025closer} where information components are expressed for one modality over the other, limiting the sharing of complementary features between the modalities.

To address these challenges, we propose an ideal fusion strategy for action anticipation that addresses both representation collapse. First, to counter feature collapse, we aim to increase the effective rank~\citep{roy2007effective} of the fused representation matrix, as it serves as a proxy for the diversity of information. This is motivated by prior work defining the effective rank as matrix entropy~\citep{von2018mathematical,renyi1961measures,zhang2023matrix}. Based on this, we develop \textit{Rank-enhancing Token fuser}, a theoretically grounded fusion framework that selectively blends \textit{less informative} channels (i.e., channels that contribute little to the principal eigenvector directions) with complementary features from another modality. Second, we posit that compatible modalities mutually increase each other's effective rank. Hence, to address modality collapse, we perform an analysis of different modalities to pair with RGB for action anticipation. We demonstrate through a relative rank analysis that the depth modality results in the most balanced feature space between both modalities. We use this intuition to develop \texttt{R3D} (Rank-enhancing fusion in 3D), the first depth-informed action anticipation architecture. We demonstrate performance improvements across a wide variety of benchmarks. Our main contributions are summarized as follows:

\begin{enumerate}
    \item \textbf{Rank-enhancing Token Fuser}: We are the first to formulate multi-modal fusion as a problem of \textit{rank-targeted fusion} to simultaneously address feature collapse and modality collapse. We provide theoretical conditions under which selective channel blending provably increases effective rank and correspondingly prevents representation collapse.

    \item \textbf{Depth-aware 3D Action Anticipation}: We present \texttt{R3D}, the first depth-informed framework for 3D action anticipation and show that depth is the most complementary modality to pair with RGB for preserving modality-specific features.

    \item \textbf{State-of-the-art performance}: \texttt{R3D} achieves up to 3.74\% performance improvement on NTURGBD, UTKinect, and DARai, setting a new benchmark for multimodal action anticipation.
\end{enumerate}

%% file: sec/2_related_works.tex
\vspace{-0.2cm}
\section{Related Works}
\label{sec:related_works}

\subsection{Action Anticipation}
The availability of large-scale video datasets~\citep{damen2022rescaling,grauman2022ego4d,kuehne2014language,li2018eye,kokilepersaud2023focal} has spurred significant progress in addressing the challenge of action anticipation. These methods can be broadly categorized into two types. First, short-term action anticipation focuses on predicting a single future action that will occur within a few seconds~\citep{fernando2021anticipating,furnari2019would,sener2020temporal}. In contrast, long-term action anticipation seeks to forecast an extended sequence of future actions from a long-range video, aiming to predict events far into the future~\citep{abu2018will,abu2021long,sener2020temporal,gong2024actfusion, huang2024multi, kim2025multi}. To predict long-term action, recent advances in action anticipation leverage multi-modal inputs to capture complementary contextual cues~\citep{lai2024human}. For instance,~\citep{zatsarynna2021multi, zhang2024multi} explicitly integrated object features alongside RGB data. ~\citep{zhong2023anticipative} leveraged audio to capture off-camera events and to disambiguate visually similar actions. ~\citep{beedu2024efficacy, kim2025multi} integrates text to provide rich and fine-grained semantic context. As humans perform actions in a 3D world, prior studies have explored the use of human pose as a complementary modality~\citep{duan2022revisiting, guan2022afe, barsoum20203d, cao2022qmednet, pmlr-v28-koppula13}, often relying on motion capture systems or pose estimation pipelines. In contrast, we are the first to introduce raw depth data as a multi-modal input without requiring additional motion capture hardware, making it practical and deployable in real-world settings where commodity RGB-D cameras are readily available. 

\subsection{Multi-modal Fusion Method}
A core challenge in multi-modal learning is the effective fusion of heterogeneous data streams~\citep{huang2020pixel, hu2021unit, akkus2023multimodal, chen2024internvl, kim2021vilt}. Fusion strategies are broadly categorized based on their methodology of fusion. Aggregation-based fusion strategies combine features through concatenations, summation, or attention pooling~\citep{valada2020self, rotondo2019action, hazirbas2016fusenet, zeng2019deep,prabhushankar2022olives}. Alignment-based fusion methods temporally or semantically align modality representations~\citep{wang2016learning, zatsarynna2021multi, zhong2023anticipative, beedu2024efficacy, kim2025multi, huang2020pixel, kim2021vilt}. Finally, hybrid approaches integrate both aggregation and alignment techniques~\citep{baltruvsaitis2018multimodal}. Despite their architectural advances, these fusion strategies often suffer from feature redundancy where redundant or weakly informative features from one modality can overwhelm the shared representation space, resulting in suboptimal fusion~\citep{liu2025tacfn}. Our method addresses this limitation by encouraging diverse representations, as evidenced by an increased effective rank in the joint embedding space.

\subsection{Approximating Information through Rank}
The technical novelty in this paper is targeted multi-modal feature fusion. We utilize the effective rank of the feature matrix as a measure of information content. \cite{von2018mathematical,renyi1961measures} identify that some measure of the uniformity of the singular value spectrum, such as effective rank \cite{roy2007effective}, is representative of the entropy of a matrix \cite{zhang2023matrix}, and its information content. Intuitively, a more uniform singular value distribution implies that the representation of the data varies along a greater number of feature directions, thus exhibiting greater diversity and countering potential representational collapse \cite{jing2021understanding}. To capture this eigenvalue uniformity as a metric, the effective rank computes the entropy of the normalized eigenspectrum of a matrix $Z$ with associated eigenvalues $[\sigma_1, ..., \sigma_r]$ as follows:
\vspace{-0.2cm}
\begin{equation}
\label{eq:1}
\mathrm{ERank}(Z) := \exp\left(-\sum_{j=1}^{\text{rank}(Z)} p_j \log p_j\right), \quad p_j = \frac{\sigma_j(Z)}{\|Z\|_*}.
\end{equation}
Due to these useful properties, measuring the rank of a representation space has gained traction in various research areas. Within self-supervised learning, \cite{thilak2023lidar,agrawal2022alpha,garrido2023rankme} use some measure of the eigenvalue spectra as a way to identify better performing models without the need for task specific fine-tuning. \cite{kokilepersaud2025adadim} show that adapting a loss during training based on rank measurements can improve downstream performance. Other works devise loss functions with the explicit goal of increasing rank \cite{bardes2021vicreg,kokilepersaud2024hex} and manipulating eigenvalues directly \cite{kim2023vne}. Despite the potential of integrating rank into training paradigms, there is limited work on applying rank as a mechanism to improve modality fusion. \cite{chaudhuri2025closer} analyze representational collapse in multi-modal architectures through effective rank measurements. However, their method does not involve directly integrating the rank and instead relies on discriminator networks. In contrast, our work is the first to demonstrate that rank information can directly inform targeted feature fusion.

%% file: sec/3_analysis.tex
\section{Analysis}
\label{sec:analysis}

\subsection{Theory: Rank-enhancing Multi-modal Fusion}
We introduce a novel \textit{Rank-enhancing Token Fuser (RTF)} that selectively blends \textit{less informative} channels with features from another modality in a way that provably increases the effective rank of the resulting representation. We show the proof in Section~\ref{lab:supp}.

\vspace{0.2cm}
\noindent \textbf{\textit{Notation.}} Let $X \in \mathbb{R}^{T \times D}$ denote the representation matrix of a modality, where $T$ is the number of timesteps and $D$ is the channel (feature) dimension. Let the singular value decomposition (SVD) of $X$ be $X = U \Sigma V^\top$, where $\Sigma = \mathrm{diag}(\sigma_1, \dots, \sigma_r)$ contains singular values in decreasing order, $V = [v_1, \dots, v_r]^\top \in \mathbb{R}^{r \times D}$ contains the corresponding right singular vectors. We refer to these right singular vectors as the principal directions of the representation. Let \( x_c \in \mathbb{R}^T \) denote the \( c \)-th column of \( X \). Similarly, let $Y \in \mathbb{R}^{T \times D}$ denote features from another modality with columns \( y_c \in \mathbb{R}^T \). These serve as candidates for enhancing \textit{less informative} channels in $X$ through selective blending.

\vspace{0.1cm}
\noindent \textbf{\textit{Definition 1. Channel Informativeness.}} We define the \textit{informativeness} of channel \( c \) as its contribution to the top-\( k \) singular vectors of the representation:  $I_c := \sum_{i=1}^k \sigma_i^2 v_{i,c}^2$, where \( v_{i,c} \) is the \( c \)-th component of the \( i \)-th right singular vector \( v_i \). Channels with low \( I_c \) are considered \textit{less informative} and are potential candidates for fusion with complementary channels from other modalities. Let \( \mathcal{C}_{\text{low}} := \{ c \in \{1, \dots, D\} \mid I_c \leq \eta \} \) denote the set of \textit{low informativeness channels}, where \( \eta > 0 \) is a predefined threshold.

\vspace{0.1cm}
\noindent \textbf{\textit{Definition 2. Fusion.}} We define a fused representation \( X' \in \mathbb{R}^{T \times D} \) by selectively blending channels in \( \mathcal{C}_{\text{low}} \) with signals from modality \( Y \):
\[
x_c' = 
\begin{cases}
\alpha_c x_c + (1 - \alpha_c) y_c & \text{if } c \in \mathcal{C}_{\text{low}}, \\
x_c & \text{otherwise},
\end{cases}
\]
where \( \alpha_c \in [0, 1] \) is a learnable channel-wise blending coefficient. The goal is to enhance underutilized directions while preserving already informative ones.

\begin{theorem}[Channel Fusion Increases Effective Rank]
\label{thm:main} Let \( u_1, \dots, u_k \in \mathbb{R}^T \) denote the top-\( k \) left singular vectors of \( X \), and define \( \delta_k := \sigma_k - \sigma_{k+1} \) as the singular value gap, which quantifies the separation between the dominant subspace (top-\( k \)) and the residual space.

\noindent Assume:
\begin{enumerate}
    \item \textbf{(Bounding condition)} For all \( c \in \mathcal{C}_{\text{low}} \), \( y_c \) is zero-mean and satisfies \( \| y_c \|_2 \leq \beta \) for some constant \( \beta > 0 \).
    
    \item \textbf{(Non-trivial modification in the representation)} There exists $\exists \epsilon > 0$ such that $\sum_{c \in \mathcal{C}_{\text{low}}} \|x_c' - x_c\|_2^2 \geq \epsilon$
    
    \item \textbf{(Dominant subspace preservation)} The updated representation should not distort the dominant subspace of \(X\). This is ensured by requiring the update to only target low-informativeness channels. Specifically, the threshold \( \eta \) is small enough to satisfy $\sqrt{\eta} \leq \min\left(\frac{\delta_k}{3\sqrt{|\mathcal{C}_{\text{low}}|}}, \frac{\epsilon}{4|\mathcal{C}_{\text{low}}|\beta}\right)$
    

    \item \textbf{(Bounded Dominant Subspace Alignment)} The injected channels $y_c$ are not perfectly aligned with the dominant subspace of $X$. Formally, the projection of each injected channel $y_c$ to the dominant subspace of $X$ \( \mathcal{U}_k = \mathrm{span}\{u_1, \dots, u_k\} \) is bounded by its projection onto the remaining subspace. i.e., $\sum_{i=1}^k |\langle y_c, u_i \rangle|^2 \leq \gamma \sum_{i=k+1}^r |\langle y_c, u_i \rangle|^2, \quad \text{where } 0 \leq \gamma < 1.$
    
\end{enumerate}
Then $y_c$ introduces novel directions in the feature space of $X$, hence the effective rank satisfies $\mathrm{ERank}(X') > \mathrm{ERank}(X)$, where the equation of $\mathrm{ERank}(X)$ follows~\eqref{eq:1}.
\end{theorem}

\begin{figure}[t!]
\begin{center}
\includegraphics[width=0.8\linewidth]{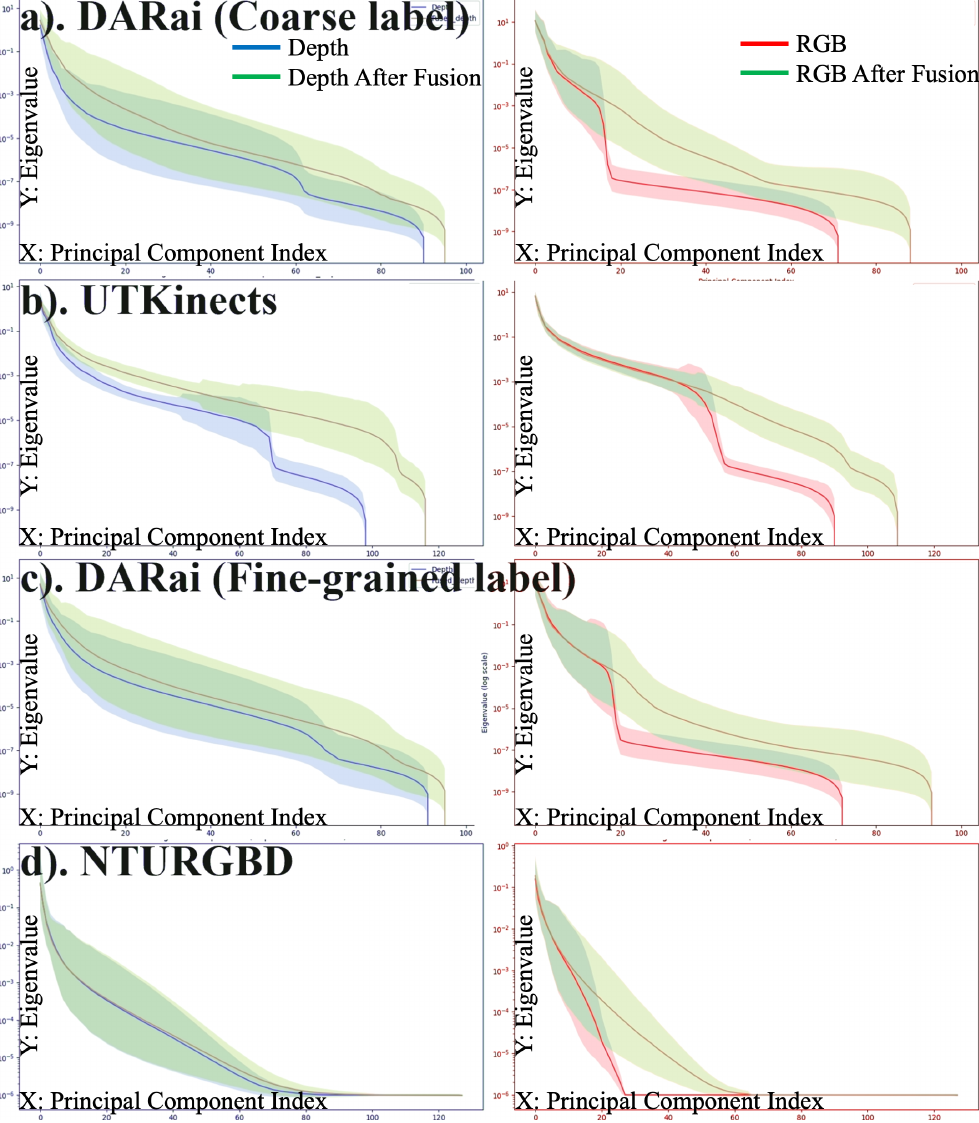}
\end{center}
\vspace{-0.5cm}
   \caption{This figure compares the eigenvalue spectra of each modality before (Depth - blue, RGB - red) and after fusion (green) using the formulation in Theorem~\ref{thm:main}. The left column shows the spectrum for the Depth modality, and the right column for RGB. Across all datasets and label granularities, the fused modality consistently exhibits a flatter spectrum in mid-to-lower components as well as the dominant ones.}
\label{fig:eigen-spectra}
\vspace{-0.5cm}
\end{figure}

In short, Theorem~\ref{thm:main} provides mathematical conditions under which selectively blending \textbf{\textit{less informative}} channels in a modality $X$ with channels from another modality $Y$ that have bounded alignment with the dominant subspace of $X$ leads to increased feature diversity, as quantified by effective rank. 

Figure~\ref{fig:eigen-spectra} shows that across all datasets, the eigenvalue spectra of both RGB and Depth representations become flatter after fusion. This flattening reflects a more uniform distribution of eigenvalues, which implies that the spectrum exhibits greater entropy. As in Equation~\ref{eq:1}, since effective rank is defined as the entropy over the spectrum, a flatter spectrum corresponds to higher effective rank, indicating a richer and more balanced representation. We provide detailed analysis of Figure~\ref{fig:eigen-spectra} in Section~\ref{sup:fig2-explained}.

Notably, this increase in effective rank occurs in both modalities, as shown in Figure~\ref{fig:eigen-spectra} (a-d). In other words, fusion not only enhances the effective rank of RGB through Depth modality, but also enhances Depth through RGB, confirming that the representational benefit is mutual. This mutual increase is desirable because it ensures that the fused representation integrates complementary information from both modalities, rather than favoring one. According to Theorem~\ref{thm:main}, such mutual improvement requires the injected features to be sufficiently complementary to the target modality, formalized by having a bounded projection onto its dominant subspace. This principle motivates the search for optimal modality combinations that maximize balanced and mutual rank enhancement.


\begin{figure}[t!]
\begin{center}
\includegraphics[width=0.7\linewidth]{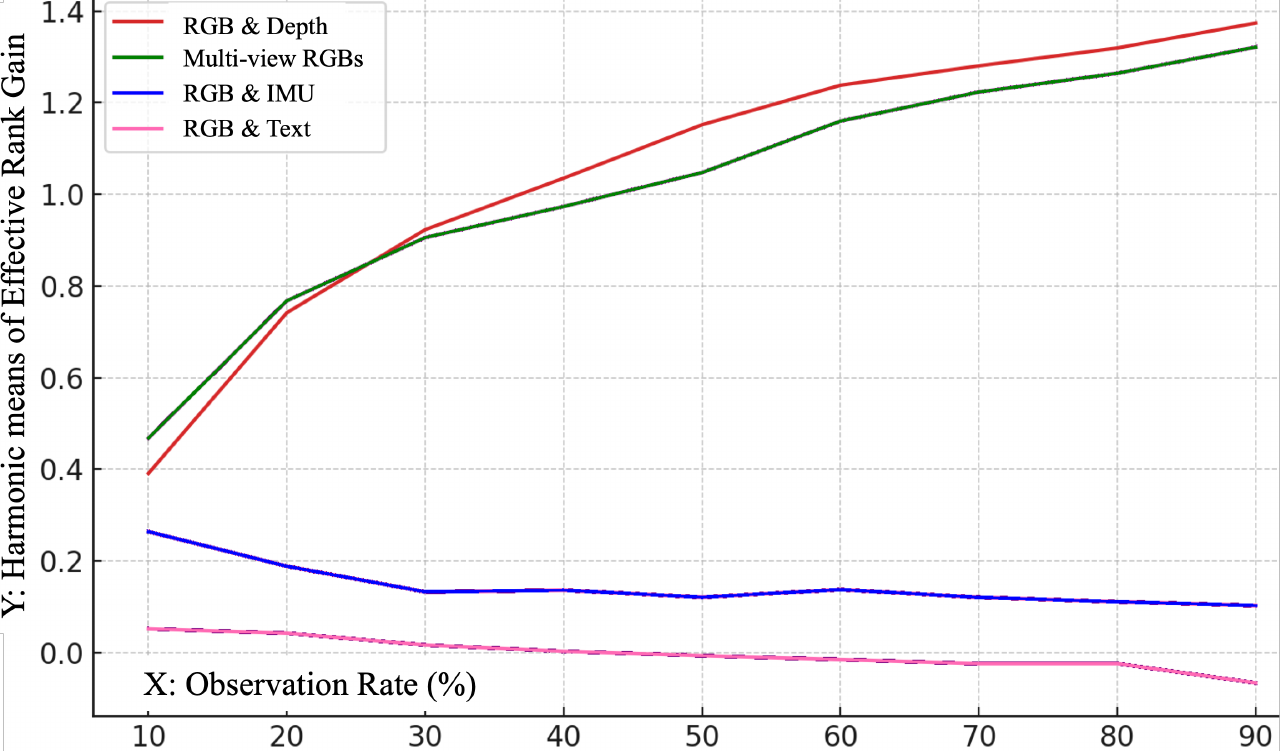}
\end{center}
\vspace{-0.5cm}
   \caption{This figure compares the Harmonic Mean of effective rank gain across four modalities: Multi-view RGB, Text, IMU, and Depth. Depth consistently achieves the highest harmonic mean across all observation rates, indicating a more balanced interaction with RGB compared to other modalities.}
\label{fig:modality-comparison}
\vspace{-0.6cm}
\end{figure}

\subsection{Modality Selection}
\label{sec:modality-selection}
In this section, we show that not all modality pairs are equally effective for preserving modality-specific features. To quantify this, we analyze how the effective rank of each modality changes after fusion with RGB across 4 different modalities: Depth, Multi-view RGB, IMU, and Text, each denoted as $\Delta_{\text{Depth}}$, $\Delta_{\text{MVRGB}}$, $\Delta_{\text{IMU}}$, $\Delta_{\text{Text}}$, respectively. In this case, all modalities are used within a fusion framework for an action anticipation task. For analysis, we extract the fused representation of each modality at an intermediate point in the network. Then, we introduce the harmonic mean of effective rank gains as our evaluation metric. This is because this score captures both the magnitude and the symmetry of increases in effective rank across both modalities. A high harmonic mean implies that both modalities gain substantially from fusion and do so in a balanced manner. Figure~\ref{fig:modality-comparison} visualizes this harmonic mean score for four modality combinations:
$\frac{2 \cdot \Delta_{\text{RGB}} \cdot \Delta_{\text{Depth}}}{\Delta_{\text{RGB}} + \Delta_{\text{Depth}}}, \frac{2 \cdot \Delta_{\text{RGB}} \cdot \Delta_{\text{MVRGB}}}{\Delta_{\text{RGB}} + \Delta_{\text{MVRGB}}}, \frac{2 \cdot \Delta_{\text{RGB}} \cdot \Delta_{\text{IMU}}}{\Delta_{\text{RGB}} + \Delta_{\text{IMU}}}, \frac{2 \cdot \Delta_{\text{RGB}} \cdot \Delta_{\text{Text}}}{\Delta_{\text{RGB}} + \Delta_{\text{Text}}}$. As shown, Depth consistently achieves the highest harmonic mean across all observation rates. This suggests that fusion with Depth yields substantial improvements in overall representational capacity (magnitude) and distributes those improvements symmetrically across both RGB and Depth modalities. In contrast, other modalities tend to enhance one modality while degrading or leaving the other unchanged. Based on this analysis, we identify Depth as the most complementary modality to RGB, as it supports mutual enhancement without collapsing the expressive structure of either representation.

%% file: sec/4_methods.tex
\begin{figure*}[t!]
\begin{center}
\includegraphics[width=0.7\linewidth]{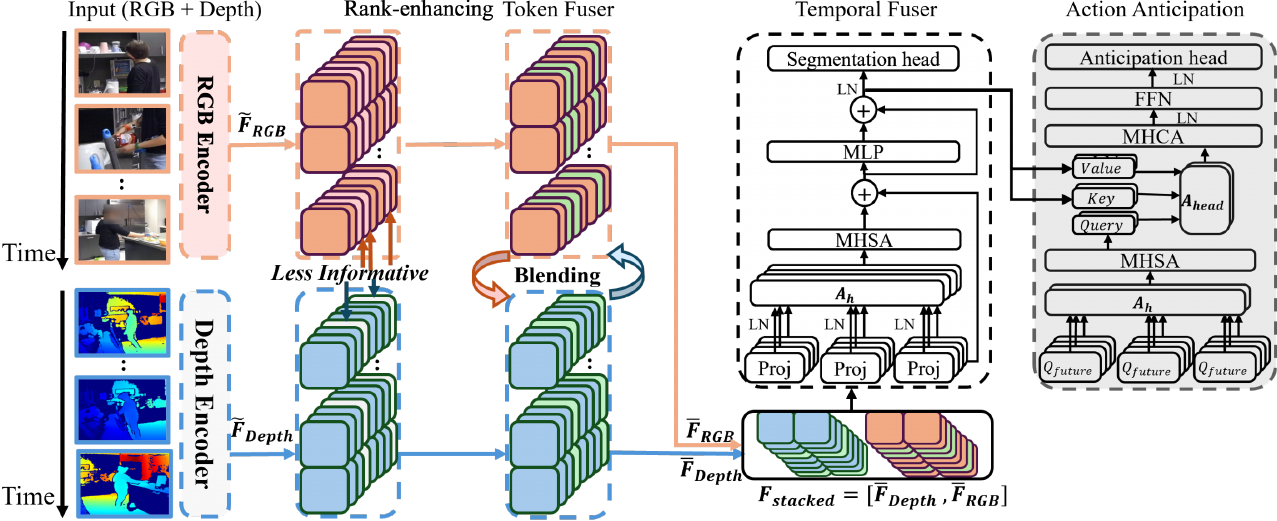}
\end{center}
\vspace{-0.5cm}
   \caption{The detailed architecture of \texttt{R3D}. It comprises three components: the Rank-Enhancing Token Fuser (\texttt{RTF}), the Temporal Fuser, and the Action Anticipation Module. The \texttt{RTF} compensates for \textbf{\textit{less informative}} channels in each modality by blending complementary information, while the Temporal Fuser captures continuous temporal dependencies and segments each timestamp. Finally, the Action Anticipation Module predicts future actions based on the integrated multi-modal information.}
\label{fig:architecture}
\vspace{-0.5cm}
\end{figure*}

\section{Methodology}
This section details the components of \texttt{R3D}. Figure \ref{fig:architecture} provides a detailed illustration of the overall architecture of \texttt{R3D}. We begin with the RGB Encoder and Depth Encoder, which extracts meaningful features from the input data (RGB video and Depth video, respectively). The Rank-enhancing Token Fuser (\texttt{RTF}) adaptively blends \textbf{\textit{less informative}} channels of both RGB and Depth modalities. Temporal Fuser integrates this multi-modal information and captures how multi-modal dependencies evolve over time for improved anticipation. Lastly, we describe the Action Anticipation Module.

\subsection{RGB and Depth Encoders}
We have a dataset consisting of $N$ paired videos consisting of RGB and Depth modalities. The RGB videos are denoted as $\{X_n^{\text{RGB}}\}_{n=1}^N$, and the corresponding Depth videos as $\{X_n^{\text{Depth}}\}_{n=1}^N$. Each RGB video $X_n^{\text{RGB}} \in \mathbb{R}^{B \times T \times H \times W \times 3} $ and Depth video $X_n^{\text{Depth}} \in \mathbb{R}^{B \times T \times H \times W \times 1} $ is encoded into visual features $\mathbf{F}^{\text{RGB}} \in \mathbb{R}^{B \times T \times C} $ and $\mathbf{F}^{\text{Depth}} \in \mathbb{R}^{B \times T \times C} $, respectively, where $B$ is the batch size, $T$ the number of frames, $H$ height, $W$ width, and $C$ the feature embedding dimension. To obtain this feature embeddings, we use a pretrained ResNet50 as the video encoder. To reduce computational cost while maintaining temporal structure, we sample frames at regular intervals using a temporal stride \( \tau \). This results in a sampled sequences $F_{\tau}^{\text{RGB}} \in \mathbb{R}^{B \times T_{\tau} \times C}$ and $F_{\tau}^{\text{Depth}} \in \mathbb{R}^{B \times T_{\tau} \times C}$, where $T_{\tau} = \left\lfloor \frac{T}{\tau} \right\rfloor$ is the number of sampled frames.
The features of the sampled frames are then passed through a linear transformation layer $W^{\text{RGB}} \in \mathbb{R}^{C \times D}$ and $W^{\text{Depth}} \in \mathbb{R}^{C \times D}$, followed by a $\text{ReLU}(\cdot)$ activation function:
\begin{equation}
\tilde{\mathbf{F}}^{\text{RGB}}  = \text{ReLU}(F_{\tau}^{\text{RGB}}W^{\text{RGB}});
\tilde{\mathbf{F}}^{\text{Depth}}  = \text{ReLU}(F_{\tau}^{\text{Depth}}W^{\text{Depth}}),
\end{equation}
where $\tilde{\mathbf{F}}^{\text{RGB}}\in \mathbb{R}^{B \times T_{\tau} \times D}$ and $\tilde{\mathbf{F}}^{\text{Depth}}\in \mathbb{R}^{B \times T_{\tau} \times D}$ represent RGB input tokens and Depth input tokens, respectively, and $D$ is a new feature dimension. The resulting features \( \tilde{\mathbf{F}}^{\text{RGB}} \) and \( \tilde{\mathbf{F}}^{\text{Depth}} \), as shown in Figure~\ref{fig:architecture}, are then fed into \texttt{RTF} for cross-modal integration.

\subsection{Rank-enhancing Token Fuser (RTF)}
In Section~\ref{sec:analysis}, we show that blending channels that are weakly aligned with the singular vectors of one modality with complementary channels from another can increase feature diversity. In this section, we translate this insight into a differentiable fusion module, \texttt{RTF}, as described in Figure~\ref{fig:architecture}.

\noindent \textbf{Channel importance estimation.} 
As illustrated in \textit{Definition 1} of Theorem~\ref{thm:main},
we quantify the \textit{importance} of each channel based on how much each channel contributes to the singular vector of the representation matrix. Given the SVD of the reshaped RGB feature matrix \( \tilde{\mathbf{F}}^{\text{RGB}} \in \mathbb{R}^{B \times T_\tau \times D} \) as \( \tilde{X}^{\text{RGB}} \in \mathbb{R}^{(B \cdot T_\tau) \times D} \), we compute \( \tilde{X}^{\text{RGB}} = U^{\text{RGB}} \Sigma^{\text{RGB}} (V^{\text{RGB}})^\top \), where the singular values are \( \sigma_1^{\text{RGB}}, \dots, \sigma_r^{\text{RGB}} \) and right singular vectors are \( v_i^{\text{RGB}} \in \mathbb{R}^D \). The \emph{channel-wise importance score} \( I_c^{\text{RGB}} \) is then defined as $I_c^{\text{RGB}} = \sum_{i=1}^r (\sigma_i^{\text{RGB}})^2 (v_{i,c}^{\text{RGB}})^2$. Similarly, we apply SVD to the reshaped Depth feature matrix \( \tilde{\mathbf{F}}^{\text{Depth}} \in \mathbb{R}^{B \times T_\tau \times D} \), resulting in the importance score denoted as $I_c^{\text{Depth}} = \sum_{i=1}^r (\sigma_i^{\text{Depth}})^2 (v_{i,c}^{\text{Depth}})^2$,
where \( I_c^{\text{RGB}},I_c^{\text{Depth}} \) reflects how much channel \( c \) contributes to the singular vector of \( \tilde{\mathbf{F}}^{\text{RGB}}, \tilde{\mathbf{F}}^{\text{Depth}}\), respectively. As illustrated in Theorem~\ref{thm:main}, channels with low \( I_c \) are considered \textit{less informative}, thus are blended with complementary features from the other modality. This targeted fusion ensures that \textit{less informative} channels in one modality are enhanced by incorporating complementary cues from the other. Thus, to identify \textit{less informative} channels, we select the bottom \( k' \) channels with the lowest importance scores:
\begin{equation}
\mathcal{I}_\text{low}^\text{RGB} = \underset{c}{\text{argmin}_{k'}} \quad I_c^\text{RGB}; \mathcal{I}_\text{low}^\text{Depth} = \underset{c}{\text{argmin}_{k'}} \quad I_c^\text{Depth}.
\end{equation}
\noindent \textbf{Adaptive channel blending.} Once the least important channels are identified, we perform cross-modal feature blending only on the bottom $k'$ least importance channels as determined by the channel importance estimation. Here, we introduce a learnable scaling parameter \(\alpha_{\mathcal{I}_\text{low}^\text{RGB}} \) and \( \alpha_{\mathcal{I}_\text{low}^\text{Depth}} \) that are randomly initialized from a uniform distribution, to enable adaptive channel blending:
\begin{equation}
\bar{\mathbf{F}}^\text{RGB}_{:, :, \mathcal{I}_\text{low}^\text{RGB}} = \alpha_{\mathcal{I}_\text{low}^\text{RGB}} \cdot y^\text{Depth}_{:, :, \mathcal{I}^\text{RGB}} + (1 - \alpha_{\mathcal{I}_\text{low}^\text{RGB}}) \cdot y^\text{RGB}_{:, :, \mathcal{I}^\text{RGB}},
\end{equation} 
\begin{equation}
\bar{\mathbf{F}}^\text{Depth}_{:, :, \mathcal{I}_\text{low}^\text{Depth}} = \alpha_{\mathcal{I}^\text{Depth}} \cdot y^\text{RGB}_{:, :, \mathcal{I}_\text{low}^\text{Depth}} + (1 - \alpha_{\mathcal{I}^\text{Depth}}) \cdot y^\text{Depth}_{:, :, \mathcal{I}_\text{low}^\text{Depth}},
\end{equation}
where the updated features, $\bar{\mathbf{F}}^\text{RGB}_{:, :, \mathcal{I}_\text{low}^\text{RGB}}$ and $\bar{\mathbf{F}}^\text{Depth}_{:, :, \mathcal{I}_\text{low}^\text{Depth}}$ denote the refined representations after adaptive channel blending for the least informative channels $\mathcal{I}_\text{low}^\text{RGB}, \mathcal{I}_\text{low}^\text{Depth}$. These parameters \(\alpha_{\mathcal{I}_\text{low}^\text{RGB}} \) and \( \alpha_{\mathcal{I}_\text{low}^\text{Depth}} \) allow the model to adaptively control the degree of feature blending between RGB and Depth channels, rather than enforcing a hard swap.

These partially updated features are then combined with the remaining channels to form the final aggregated features: $\bar{\mathbf{F}}^{\text{RGB}}$ and $\bar{\mathbf{F}}^{\text{Depth}}\in \mathbb{R}^{B \times T_{\tau} \times D}$. Finally, the fused features $\bar{\mathbf{F}}^{\text{RGB}}$ and $\bar{\mathbf{F}}^{\text{Depth}}$ are stacked together along the channel dimension ($\mathbf{F}_{\text{stacked}}$ in Figure~\ref{fig:architecture}) and passed as an input to the Temporal Fuser for sequential modeling.

\subsection{Temporal Fuser and Action Anticipation}
As illustrated in Figure~\ref{fig:architecture}, Temporal Fuser takes the output from \texttt{RTF} and captures sequential dependencies over time, to accurately interpret the meaning of each frame. Its primary goal is to integrate temporal information so that the model can determine the meaning of a scene at every time step. Section~\ref{lab:temporal-fuser} details the three components of Temporal Fuser: MHSA, MLP, and Segmentation Head.

Building on these temporal representations, Action Anticipation Module predicts future human actions. Section~\ref{lab:action-anticipation-module-arc} describes its main components: MHCA, FFN, Action Anticipation Head.


%% file: sec/5_experiments.tex
\input{table/main-experiment}

\section{Experiments}
This section presents experimental analyses and discussion of the proposed method. The detailed dataset configuration is described in Section~\ref{lab:datasets}, the experimental setup is detailed in Section~\ref{lab:setup}, and the computational cost and scalability are detailed in Section~\ref{lab:comp-cost-scalability}. For all analyses, we set the observation rate $\alpha \in \{0.2, 0.3\}$ and prediction rate $\beta \in \{0.1, 0.2, 0.3, 0.5\}$, and report mean over classes (MoC) accuracy as the evaluation metric. Qualitative analyses are provided in Section~\ref{lab:qualitative}.

\subsection{Quantitative Analysis}
\noindent\textbf{State-of-the-art comparison.}
As we observe in Table \ref{tab:main-experiment}, \texttt{R3D} consistently outperforms the state-of-the-art across all datasets and evaluation settings, achieving an average accuracy gain of up to 3.74\%. We compare the performance of multiple architectures for action anticipation and observe that our approach achieves the best performance among all models, with AFFT \cite{zhong2023anticipative}, m\&m-Ant \cite{kim2025multi}, FUTR \cite{gong2022future}, and GTAN~\citep{zatsarynna2024gated}, serving as the state-of-the-art baselines. This gain is particularly notable in scenarios with lower observation rates ($\alpha = 0.2$), where depth information allows the model to leverage subtle cues even when visual input is limited. The performance boost is more pronounced on the DARai dataset than on the UTKinects dataset. This is because DARai contains more dense action sequences and more frequent transitions. In this case, depth information enhances the model’s ability to capture nuanced temporal variations, leading to greater improvements compared to other datasets. In contrast, on NTURGBD, the strength of \texttt{R3D} is less pronounced. This can be attributed to the fact that NTURGBD provides high-quality RGB-Depth pairs with well-trimmed and well-curated action segments. While \texttt{R3D} is designed to be robust in less curated, noisy real-world settings (See~\ref{lab:robustness}), its relative advantage diminishes when both modalities are already strong, which is less common in real-world settings.

\input{table/ablation-token}
\input{table/ablation-alphalearning}
\begin{figure}[t!]
\begin{center}
\includegraphics[width=1\linewidth]{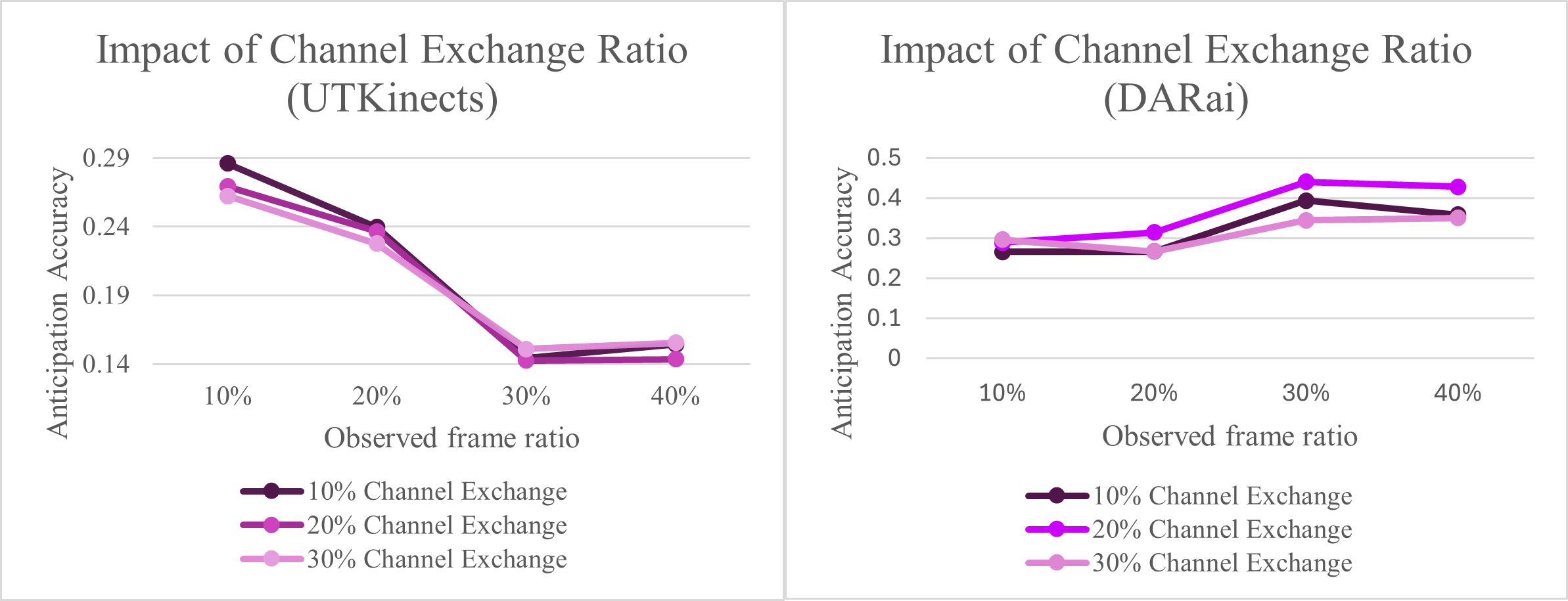}
\end{center}
\vspace{-0.5cm}
   \caption{Ablation study on UTKinects and DARai dataset examining the impact of the proportion of exchanged channels (10\%, 20\%, 30\%) in Token Fuser.}
\label{fig:ablation-channelexchange}
\vspace{-0.3cm}
\end{figure}

\subsection{Ablation study}
We perform ablation studies to assess the impact of depth modality (Table \ref{tab:main-experiment}), impact of Rank-enhacing Token Fuser (Table \ref{tab:ablation-token}), Rank-enhacing Token Fuser modeling (Figure \ref{fig:ablation-channelexchange}, Table \ref{tab:ablation-alphalearning}), and impact of Temporal Fuser (Table \ref{tab:ablation-temporal}). We show the ablation study of Temporal Fuser and additional Rank-enhacing Token Fuser modeling in Section~\ref{sup:ablation-study} and Section~\ref{sup:ablation-latent}, respectively.

\noindent\textbf{RGB only.} First, we evaluate the effect of incorporating depth information by comparing models trained with and without depth features. As shown in Table \ref{tab:main-experiment}, integrating depth consistently improves performance across all datasets, demonstrating its importance for capturing fine-grained spatial relationships and motion directionality.

\noindent\textbf{RGB-Depth aggregation without Rank-enhancement.} Next, we analyze the impact of the Token Fuser by removing it from \texttt{R3D}. The results in Table \ref{tab:ablation-token} show that models without Rank-enhancing Token Fuser exhibit lower performance, confirming its effectiveness in facilitating cross-modal information exchange and improving representation quality.

\noindent\textbf{Role of adaptive $\alpha$ during fusion.} To further investigate Rank-enhancing Token Fuser modeling, we compare static channel exchange with adaptive channel exchange method as shown in Table \ref{tab:ablation-alphalearning}. Across all datasets, adaptive blending yields better performance, as action anticipation task exhibits high variability in object interactions and spatial configurations, requiring a flexible fusion mechanism.

\noindent\textbf{Impact of varying channel exchange ratio.} Additionally, we conduct an ablation study on the optimal channel exchange ratio for Token Fuser as shown in Figure \ref{fig:ablation-channelexchange}. Our findings reveal that the optimal exchange ratio varies by dataset: UTKinect achieves the highest performance at 10\% channel exchange, while DARai performs best at 20\% exchange. This suggests that datasets with more structured, repetitive actions (e.g., UTKinect) benefit from lower exchange rates, whereas datasets with more complex, unstructured activities (e.g., DARai) require a higher degree of adaptive feature replacement.

\subsection{More SOTA comparisons}
\input{table/ablation-modality}
\noindent \textbf{More modality pairs.} 
To further validate the general effectiveness of \texttt{R3D} with different modalities, as shown in Table~\ref{tab:ablation-modality}, we additionally present the State-of-the-art comparisons of action anticipation performance across four different modality pairs: Depth, Multi-view RGB, IMU, and Text at varying observation rates ($\alpha$) and prediction rates ($\beta$). For comparison, we use AFFT~\citep{zhong2023anticipative}, the current state-of-the-art method in multi-modal action anticipation. To ensure fairness and reproducibility, we conduct three runs with fixed random seeds (1, 10, 13452) and simply replace the input from the Figure~\ref{fig:architecture}. Overall, \texttt{R3D} achieves superior performance to AFFT for modality pairs with higher harmonic means as illustrated in Figure~\ref{fig:modality-comparison}. For RGB - IMU pair, \texttt{R3D} underperforms compared to AFFT. We attribute this to the lack of sufficiently balanced interaction between RGB and IMU features, as also reflected in Figure~\ref{fig:modality-comparison}. We expect that tailoring the fusion parameters for RGB - IMU would enhance the interaction and further improve performance.

\input{table/action-recognition}
\noindent \textbf{More downstream tasks.} 
To further validate the general effectiveness of \texttt{RTF}, we additionally apply our method to Action Segmentation task. Unlike Action Anticipation, Action Segmentation task segments an input video into temporally contiguous segments, assigning a categorical action label to each segment independently~\citep{lea2017temporal, ishikawa2021alleviating, farha2019ms, wang2020boundary, yi2021asformer, xu2022don, behrmann2022unified, liu2023diffusion, gong2024actfusion}. To comprehensively evaluate the task, first, we assess the correctness of predictions at every individual frame by reporting frame-wise accuracy~\citep{lea2017temporal, ishikawa2021alleviating, wang2020boundary, xu2022don, behrmann2022unified, gong2024actfusion}. Additionally, we also measure how well the predicted segments align with the true temporal boundaries and labels of actions by measuring an edit score, a segment-wise metrics~\citep{farha2019ms, yi2021asformer, liu2023diffusion, gong2024actfusion}. To obtain results from the action segmentation task, we use the output of the Segmentation Head of the Temporal Fuser of \texttt{R3D}, illustrated in Figure~\ref{fig:architecture} and explained in detail in Section~\ref{lab:temporal-fuser}. To ensure consistency, we report the average performance over three runs with fixed random seeds (1, 10, 13452).

As shown in Table~\ref{tab:action-recognition}, \texttt{R3D} consistently outperforms the state-of-the-art across different levels of granularity. The performance gap is more pronounced on the fine-grained setting, since the finer temporal boundaries and larger number of action segments make boundary alignment substantially more challenging, amplifying the advantage of \texttt{R3D} in capturing temporal structure and maintaining segment consistency. It is worth noting that \texttt{R3D} is not originally designed for action segmentation task, yet it still demonstrates strong performance on this task.

\subsection{Robustness}
\label{lab:robustness}

\begin{figure}[t!]
\begin{center}
\includegraphics[width=1\linewidth]{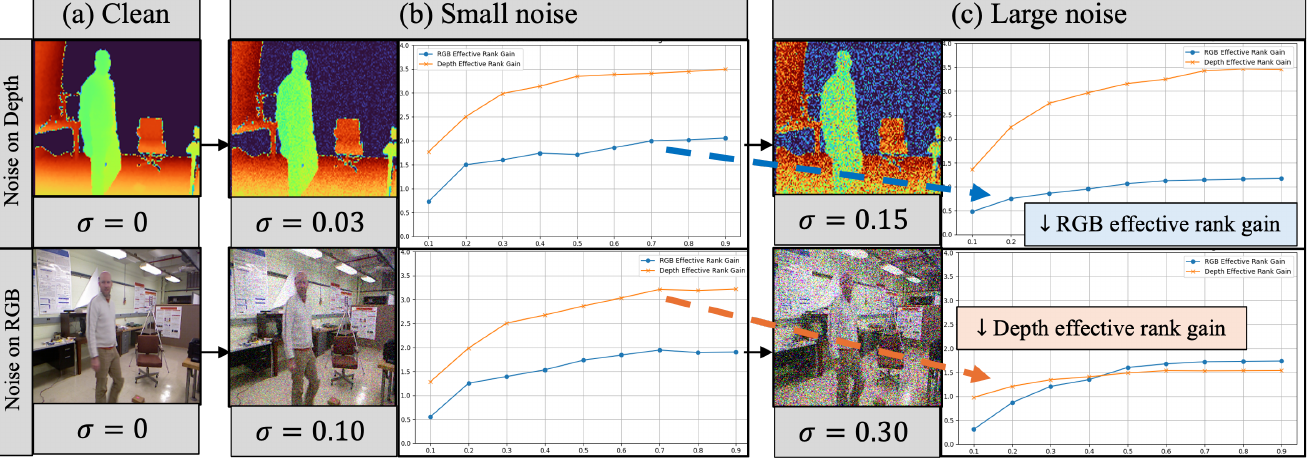}
\end{center}
\vspace{-0.5cm}
   \caption{Effective rank gain under various noise perturbations (Depth: (b) small $\sigma$=0.03, (c) large $\sigma$=0.15, (b) RGB: small $\sigma$=0.10, (c) large $\sigma$=0.30). When more noise is added to one modality, the noisy modality's effective rank gain remains stable, whereas the other (clean) modality's gain diminishes. For example, increasing noise in Depth keeps its gain stable but reduces RGB's. This shows that fusion adaptively relies more on the cleaner modality.}
\label{fig:noise-erank}
\vspace{-0.3cm}
\end{figure}

\begin{figure}[t!]
\begin{center}
\includegraphics[width=0.8\linewidth]{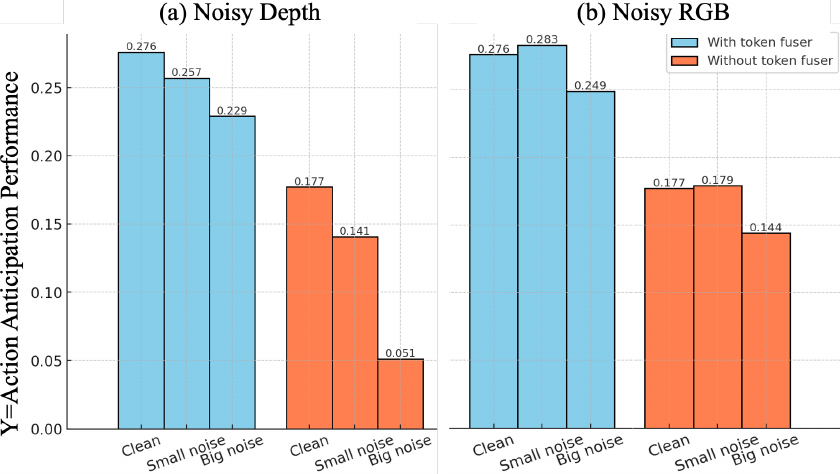}
\end{center}
\vspace{-0.5cm}
   \caption{This shows an action anticipation performance under varying noise levels (clean, small noise, large noise) in two settings: with Rank-Enhancing Token Fuser (RTF) and without RTF. Notably, (a) when noise is added to Depth, performance without RTF degrades substantially, while with RTF maintains stability, compared to (b) when noise is added to RGB.}
\label{fig:noise-performance}
\vspace{-0.7cm}
\end{figure}

In this section, we simulate the real world scenario where one modality can be unreliable or noisy. Figure~\ref{fig:noise-erank} illustrates two cases: noise added to the Depth modality (top) and to the RGB (bottom). We assess how much each modality contributes to fusion by measuring the effective rank gain across three conditions: clean (Figure~\ref{fig:noise-erank} (a)), small noise (Figure~\ref{fig:noise-erank} (b)), and large noise (Figure~\ref{fig:noise-erank} (c)). To ensure fairness, we calibrate the noise level $\sigma$ separately for Depth and RGB so that perturbations induce the comparable change in each modality's representation space.

Figure~\ref{fig:noise-performance} (a) shows that without \texttt{RTF}, the performance drops sharply when Depth is noisy, whereas with \texttt{RTF}, the performance remains stable under both noisy Depth and noisy RGB. Figure~\ref{fig:noise-erank} shows the mechanism of how \texttt{RTF} works. As noise increases (Figure~\ref{fig:noise-erank} (b) $\rightarrow$ (c)) in one modality, \texttt{RTF} adaptively shifts reliance toward the cleaner modality. For example, when Depth is corrupted, Depth's effective rank gain remains stable whereas RGB's gain diminishes (Figure~\ref{fig:noise-erank} (b) $\rightarrow$ (c), Top part), indicating that the fusion compensates by down-weighting the unreliable Depth contribution and leveraging the complementary information from RGB, so that the overall representation remains robust despite the corruption.

%% file: table/main-experiment.tex
\begin{table}[t!]
\centering
\scalebox{0.53}{
\renewcommand{\arraystretch}{1.2}
\begin{tabular}{@{}lllcccccccc@{}}
\toprule
\multirow{2}{*}{Dataset} & \multirow{2}{*}{Input}& \multirow{2}{*}{Methods}  & \multicolumn{4}{c}{$\beta (\alpha=0.2)$} & \multicolumn{4}{c}{$\beta (\alpha=0.3)$} \\  
\cmidrule(lr){4-7} \cmidrule(lr){8-11} 
 & & & 0.1 & 0.2 & 0.3 & 0.5 & 0.1 & 0.2 & 0.3 & 0.5  \\  
\midrule
\multirow{6}{*}{\shortstack{DARai~\citep{kaviani2025hierarchical} \\ (Coarse)}}  
& Uni & RNN \cite{abu2018will} & 22.40 & 22.59 & 20.71 & 18.38 & 30.75 & 25.34 & 25.99 & 23.17 \\  
&  & CNN \cite{abu2018will}& 13.28 & 13.70 & 12.99 & 13.90 & 19.85 & 16.68 & 18.69 & 17.48 \\  
&  & FUTR \cite{gong2022future} & 25.05 & 25.11 & 24.48 & 23.18 & 40.71 & 33.57 & 33.43 & 30.79 \\ 
&  & GTAN \cite{zatsarynna2024gated}& 27.70 & 29.13 & 27.35 & 26.11 & 42.47 & \textbf{42.45} & 42.27 & 34.37 \\ 
& & R3D (Uni) & 31.52 & 30.52 & 28.48 & 26.66 & 44.01 & 37.10 & 37.58 & 34.44 \\  
 \cmidrule(lr){2-11}
& Multi & AFFT \cite{zhong2023anticipative}& 23.14 & 24.78 & 23.62 & 21.02 & 33.82 & 29.25 & 28.33 & 25.45 \\
&  & m\&m-Ant \cite{kim2025multi} & 25.75 & 25.70 & 24.24 & 23.12 & 42.00 & 34.71 & 34.49 & 31.34 \\ 
 
\rowcolor{Gray} & &  \textbf{R3D (Ours)} & \textbf{33.44} & \textbf{32.14} & \textbf{30.56} & \textbf{29.59} & \textbf{46.29} & 42.05 & \textbf{43.41} & \textbf{40.25} \\  
 \midrule
\multirow{6}{*}{\shortstack{DARai~\citep{kaviani2025hierarchical} \\ (Fine-grained)}}  
& Uni & RNN \cite{abu2018will} & 7.45 & 5.69 & 4.66 & 3.58 & 9.07 & 7.35 & 6.40 & 5.04 \\  
&  & CNN \cite{abu2018will} & 7.70 & 6.28 & 5.38 & 3.98 & 6.69 & 5.48 & 5.10 & 3.86 \\  
&  & FUTR \cite{gong2022future} & 17.09 & 12.34 & 10.39 & 8.18 & 18.59 & 15.83 & 13.26 & 10.81 \\ 
 &  & GTAN \cite{zatsarynna2024gated} & 25.66 & 24.06 & 22.60 & \textbf{22.05} & 26.30 & 25.43 & 22.40 & \textbf{18.88} \\
 & & R3D (Uni) & 20.12 & 18.49 & 16.11 & 13.28 & 24.62 & 22.93 & 22.78 & 18.68 \\ 
 \cmidrule(lr){2-11}
& Multi & AFFT \cite{zhong2023anticipative} & 14.13 & 13.38 & 13.49 & 9.53 & 17.01 & 17.26 & 15.18 & 10.91 \\
\rowcolor{Gray} &  & \textbf{R3D (Ours)}& \textbf{32.57} & \textbf{25.92} & \textbf{24.02} & 16.68 & \textbf{26.97} & \textbf{28.43} & \textbf{24.81} & 18.02 \\  
\midrule
\multirow{6}{*}{UTKinects~\citep{xia2012view}}  
& Uni& RNN \cite{abu2018will} & 24.42 & 15.37 & 15.92 & 11.10 & 15.22 & 15.19 & 13.12 & 9.89 \\  
& & CNN \cite{abu2018will} & 25.00 & 16.60 & 14.79 & 11.07 & 15.10 & 16.67 & 14.29 & 8.08 \\ 
& & FUTR \cite{gong2022future} & 29.63 & 20.75 & 20.53 & 16.48 & 16.46 & 16.90 & 14.31 & 9.95 \\ 
& & GTAN \cite{zatsarynna2024gated} & 34.38 & 35.71 & \textbf{39.54} & 34.17 & 28.27 & 32.36 & 27.22 & 21.16 \\
& & R3D (Uni) & 29.63 & 20.75 & 20.53 & 16.48 & 16.46 & 16.90 & 14.31 & 9.95 \\  
\cmidrule(lr){2-11}
& Multi & AFFT \cite{zhong2023anticipative} & 25.00 & 16.67 & 16.67 & 12.50 & 16.67 & 16.67 & 14.29 & 10.00 \\ 
\rowcolor{Gray} & &  \textbf{R3D (Ours)} & \textbf{38.96} & \textbf{37.68} & 37.16 & \textbf{35.95} & \textbf{34.80} & \textbf{32.51} & \textbf{36.03} & \textbf{26.98} \\  
\midrule
\multirow{6}{*}{NTURGBD~\citep{shahroudy2016ntu}}  
&Uni & RNN \cite{abu2018will} 
 & 10.98 & 10.49 & 10.53 & 10.94 & 8.60 & 9.55 & 10.11 & 11.69 \\  
& & CNN \cite{abu2018will}  & 15.11 & 14.91 & 14.33 & 15.12 & 15.09 & 15.49 & 15.52 & 16.74 \\  
& & FUTR \cite{gong2022future}  & 18.59 & 18.56 & 18.61 & 18.61 & 20.04 & 20.13 & 20.07 & 20.11 \\
&  & GTAN \cite{zatsarynna2024gated}  & 18.93 & 18.51 & 18.55 & 18.70 & 21.63 & \textbf{21.72} & \textbf{21.95} & 20.93 \\
& & R3D (Uni) & 10.24 & 10.14 & 10.18 & 10.15 & 10.91 & 10.97 & 10.92 & 10.96 \\  
\cmidrule(lr){2-11}
& Multi & AFFT \cite{zhong2023anticipative} & 20.55 & \textbf{20.76} & \textbf{20.66} & \textbf{20.79} & 20.72 & 20.72 & 20.74 & \textbf{21.27} \\ 
\rowcolor{Gray} & &  \textbf{R3D (Ours)} & \textbf{21.98} & 20.18 & 19.72 & 18.90 & \textbf{23.17} & 21.27 & 20.51 & 19.89 \\  
\bottomrule 
\end{tabular}}
\caption{State-of-the-art comparisons of action anticipation performance across three widely used datasets under varying conditions of $\alpha$ (observation rate) and $\beta$ (prediction rate). ``Uni" denotes Unimodal input and ``Multi" indicates a Multi-modal setting. ``R3D (Uni)" refers to the R3D model trained solely with RGB input, without depth modality. }
\label{tab:main-experiment}
\end{table}

%% file: table/ablation-token.tex
\begin{table}[t!]
\centering
\scalebox{0.5}{
\renewcommand{\arraystretch}{1.2}
\begin{tabular}{@{}llcccccccc@{}}
\toprule
\multirow{2}{*}{Dataset} & \multirow{2}{*}{Methods} & \multicolumn{4}{c}{$\beta (\alpha=0.2)$} & \multicolumn{4}{c}{$\beta (\alpha=0.3)$} \\  
\cmidrule(lr){3-6} \cmidrule(lr){7-10} 
 &  & 0.1 & 0.2 & 0.3 & 0.5 & 0.1 & 0.2 & 0.3 & 0.5 \\  
\midrule
\multirow{2}{*}{UTKinects}  
 & W/O RTF & 29.63 & 20.75 & 20.53 & 16.48 & 16.46 & 16.90 & 14.31 & 9.95 \\  
 & W/ RTF & \textbf{38.96} & \textbf{37.68} & \textbf{37.16} & \textbf{35.95} & \textbf{34.80} & \textbf{32.51} & \textbf{36.03} & \textbf{26.98} \\ 
\midrule
\multirow{2}{*}{NTURGBD}  
 & W/O RTF & 16.86 & 16.74 & 16.74 & 16.74 & 19.53 & 19.89 & 19.74 & 19.89 \\  
 & W/ RTF & \textbf{21.98} & \textbf{20.18} & \textbf{19.72} & \textbf{18.90} & \textbf{23.17} & \textbf{21.27} & \textbf{20.51} & \textbf{19.89} \\ 
\midrule
\multirow{2}{*}{DARai}  
 & W/O RTF & 31.84 & 31.40 & 28.88 & 27.06 & 35.32 & 31.98 & 32.83 & 31.91 \\  
 & W/ RTF & \textbf{33.44} & \textbf{32.14} & \textbf{30.56} & \textbf{29.59} & \textbf{46.29} & \textbf{42.05} & \textbf{43.41} & \textbf{40.25} \\  
\bottomrule
\end{tabular}}
\caption{Ablation Study evaluating the impact of \texttt{RTF}.}
\label{tab:ablation-token}
\vspace{-0.3cm}
\end{table}

%% file: table/ablation-alphalearning.tex
\begin{table}[t!]
\centering
\scalebox{0.5}{
\renewcommand{\arraystretch}{1.2}
\begin{tabular}{@{}llcccccccc@{}}
\toprule
\multirow{2}{*}{Dataset} & \multirow{2}{*}{Methods} &  \multicolumn{4}{c}{$\beta (\alpha=0.2)$} 
& \multicolumn{4}{c}{$\beta (\alpha=0.3)$} \\  
\cmidrule(lr){3-6} \cmidrule(lr){7-10}
 &  & 0.1 & 0.2 & 0.3 & 0.5 & 0.1 & 0.2 & 0.3 & 0.5 \\  
\midrule
\multirow{2}{*}{UTKinects}  
 & Static & \textbf{40.30} & 25.62 & 26.42 & 19.27 & 13.55 & 19.10 & 16.46 & 12.17 \\  
 & Adaptive  & 38.96 & \textbf{37.68} & \textbf{37.16} & \textbf{35.95} &
 \textbf{34.80} &
 \textbf{32.51} &
 \textbf{36.03} &
 \textbf{26.98}\\ 
\midrule
\multirow{2}{*}{NTURGBD}  
 & Static  & 18.93 & 17.93 & 16.90 & 16.77 & 18.48 & 17.83 & 16.65 & 16.30\\  
 & Adaptive & \textbf{21.98} & \textbf{20.18} & \textbf{19.72} & \textbf{18.90} & \textbf{23.17} & \textbf{21.27} & \textbf{20.51} & \textbf{19.89} \\   
\midrule
\multirow{2}{*}{DARai}  
 & Static & 27.86 & 26.50 & 24.62 & 23.39 & 41.10 & 35.35 & 37.00 & 35.01 \\  
 & Adaptive & \textbf{33.44} & \textbf{32.14} & \textbf{30.56} & \textbf{29.59} & \textbf{46.29} & \textbf{42.05} & \textbf{43.41} & \textbf{40.25} \\  
\bottomrule
\end{tabular}}
\caption{Ablation Study on channel exchange method across three datasets. ``Static'' indicates Static Channel Exchange and ``Adaptive'' indicates Adaptive Channel Exchange with learnable scaling factor $\alpha_{\text{I}}$.}
\label{tab:ablation-alphalearning}
\vspace{-0.5cm}
\end{table}

%% file: table/ablation-modality.tex
\begin{table}[t!]
\centering
\scalebox{0.5}{
\renewcommand{\arraystretch}{1.2}
\begin{tabular}{@{}llcccccccc@{}}
\toprule
\multirow{2}{*}{Modality pair} & \multirow{2}{*}{Methods} & \multicolumn{4}{c}{$\beta (\alpha=0.2)$} & \multicolumn{4}{c}{$\beta (\alpha=0.3)$} \\  
\cmidrule(lr){3-6} \cmidrule(lr){7-10} 
 &  & 0.1 & 0.2 & 0.3 & 0.5 & 0.1 & 0.2 & 0.3 & 0.5 \\  
\midrule
\multirow{2}{*}{RGB - Depth} & 
AFFT~\citep{zhong2023anticipative} & 14.13 & 13.38 & 13.49 & 9.53 & 17.01 & 17.26 & 15.18 & 10.91 \\
 & \textbf{R3D (Ours)} & \textbf{32.57} & \textbf{25.92} & \textbf{24.02} & \textbf{16.68} & \textbf{26.97} & \textbf{28.43} & \textbf{24.81} & \textbf{18.02} \\ 
\midrule
\multirow{2}{*}{Different-view RGBs} & 
AFFT~\citep{zhong2023anticipative} & 11.53 & 8.46 & 7.47 & 5.84 & 10.61 & 8.51 & 8.22 & 6.88 \\  
 &  \textbf{R3D (Ours)} & \textbf{26.99} & \textbf{23.80} & \textbf{20.73} & \textbf{21.50} & \textbf{35.02} & \textbf{27.37} & \textbf{27.19} & \textbf{22.62} \\
 \midrule
\multirow{2}{*}{RGB - IMU} & 
AFFT~\citep{zhong2023anticipative} & 11.89 & \textbf{10.04} & \textbf{8.81} & \textbf{7.01} & \textbf{11.31} & \textbf{9.72} & \textbf{9.64} & \textbf{7.59} \\  
 
 & \textbf{R3D (Ours)} & \textbf{12.08} & 9.76 & 7.87 & 5.80 & 10.61 & 8.46 & 7.54 & 5.75 \\  
\midrule
\multirow{2}{*}{RGB - Text} & 
AFFT~\citep{zhong2023anticipative} & 20.21 & 21.83 & 21.93 & 17.07 & 25.79 & 25.62 & 24.90 & 19.62 \\
 &\textbf{R3D (Ours)} & \textbf{30.24} & \textbf{28.25} & \textbf{22.80} & \textbf{18.01} & \textbf{28.17} & \textbf{27.69} & \textbf{25.69} & \textbf{22.93} \\ 
\bottomrule 
\end{tabular}}
\caption{State-of-the-art comparisons of action anticipation performance on DARai dataset across four different modalities: Depth, Multi-view RGB, IMU, and Text.}
\label{tab:ablation-modality}

\end{table}

%% file: table/action-recognition.tex
\begin{table}[t!]
\centering
\scalebox{0.53}{
\renewcommand{\arraystretch}{1.2}
\begin{tabular}{@{}llcc@{}}
\toprule
\multirow{2}{*}{Dataset} & \multirow{2}{*}{Methods} & \multicolumn{1}{c}{Frame-wise metrics} & \multicolumn{1}{c}{Segment-wise metrics} \\  
\cmidrule(lr){3-3} \cmidrule(lr){4-4} 
 &  & Accuracy & Edit score \\
\midrule
\multirow{2}{*}{\shortstack{DARai~\citep{kaviani2025hierarchical} \\ (Coarse)}}  
&  ActFusion \cite{gong2024actfusion} & 31.38 & 26.62  \\
&  \textbf{R3D (Ours)} & \textbf{33.32} & \textbf{31.21} \\  
 \midrule
\multirow{2}{*}{\shortstack{DARai~\citep{kaviani2025hierarchical} \\ (Fine-grained)}} 
&  ActFusion \cite{gong2024actfusion} & 17.83 & 18.16 \\
& \textbf{R3D (Ours)} & \textbf{20.79} & \textbf{28.98} \\  
 \midrule
\multirow{2}{*}{NTURGBD~\citep{shahroudy2016ntu}}
& ActFusion \cite{gong2024actfusion}& 13.26 & 6.00 \\
& \textbf{R3D (Ours)} & \textbf{13.26} & \textbf{8.93} \\  
\bottomrule
\end{tabular}}
\caption{State-of-the-art comparisons of action segmentation performance across multiple datasets.}
\label{tab:action-recognition}
\end{table}

%% file: sec/6_conclusion.tex
\section{Conclusion}
\label{sec:conclusion}

We present \texttt{R3D}, a depth-informed framework for action anticipation that addresses two representation collapse in multi-modal fusion. The utility of effective rank as a computational means to counter both feature and modality collapse jointly is emphasized.  We introduce \texttt{RTF}, a theoretically grounded mechanism that selectively blends \textit{less informative} channels with complementary signals from another modality, provably increasing representational diversity. The principle of mutual effective rank gain, whereby the effective rank of both modalities increases simultaneously is highlighted. For the application of action anticipation, depth adds background context and directionality to the object-centric visual cues from RGB. \texttt{R3D} is validated extensively across NTURGBD, UTKinect, and DARai. Our method not only outperforms state-of-the-art approaches but also provides new insights into what constitutes effective and balanced multi-modal fusion.


%% file: sec/7_supplementary.tex
\appendix
\section{Proof of Theorem \ref{thm:main}}
\label{lab:supp}

\begin{lemma}[Stability of Dominant Subspace; see Theorem 1 in~\citep{o2023matrices}]
\label{lem:stability}
Let $X, X' \in \mathbb{R}^{T \times D}$ with $\Delta := X' - X$, and $\delta_k := \sigma_k(X) - \sigma_{k+1}(X) > 0$. If $\|\Delta\|_2 < \delta_k$, then:
\begin{equation}
\|\sin \Theta(\mathcal{U}_k, \mathcal{U}_k')\| \leq \frac{\|\Delta\|_2}{\delta_k - \|\Delta\|_2}
\end{equation}
where $\mathcal{U}_k, \mathcal{U}_k'$ are top-$k$ left singular subspaces of $X$ and $X'$.
\end{lemma}

\begin{proof}
\noindent \textbf{Step 1: Perturbation bounds.} In this step, we derive an upper bound for the spectral norm of the perturbation matrix $\Delta$. We aim to confirm that the perturbation is small enough to satisfy the conditions of the subspace stability (Lemma~\ref{lem:stability}).
Let $\Delta_c = (1-\alpha_c)(y_c - x_c)$ for $c \in \mathcal{C}_{\text{low}}$, else 0. By the triangle inequality,
\begin{equation}
\|\Delta_c\|_2 \leq |1-\alpha_c|(\|y_c\|_2 + \|x_c\|_2)
\end{equation}
From the condition $|1-\alpha_c| \leq 1$, and \textit{Assumption 1} of Theorem~\ref{thm:main} ($\|y_c\|_2 \leq \beta$):
\begin{equation}
\label{eq:7}
\|\Delta_c\|_2 \leq |1-\alpha_c|(\|y_c\|_2 + \|x_c\|_2) \leq \beta + \sqrt{I_c} \leq \beta + \sqrt{\eta}  
\end{equation}

Since the Frobenius norm of a matrix $\Delta$ is the square root of the sum of the squared L2 norms of its column vectors ($\Delta_c$), and the perturbation $\Delta$ is only non-zero for columns in the set $\mathcal{C}_{\text{low}}$:
\begin{equation}
\label{eq:9}
\|\Delta\|_F^2 = \sum_{c \in \mathcal{C}_{\text{low}}} \|\Delta_c\|_2^2
\end{equation}
By Equation~\eqref{eq:7} and Equation~\eqref{eq:9}:
\begin{equation}
\|\Delta\|_F^2 \leq \sum_{c \in \mathcal{C}_{\text{low}}} (\beta + \sqrt{\eta})^2
\end{equation}
Thus, with \textit{Assumption 3}:
\begin{equation}
\|\Delta\|_F \leq \sqrt{|\mathcal{C}_{\text{low}}|}(\beta + \sqrt{\eta}) < \delta_k/2   
\end{equation}

Since $\|\Delta\|_2 \leq \|\Delta\|_F$, we have:
\begin{equation}
\label{eq:11}
\|\Delta\|_2 < \delta_k/2   
\end{equation}

\noindent \textbf{Step 2: Subspace stability.} In this step, we use the perturbation bound from Step 1 to show that the dominant singular subspace of the data remains stable. This confirms that our fusion process does not corrupt the principal components of the original data. 
By Equation~\eqref{eq:11} and Lemma \ref{lem:stability}:
\begin{equation}
\|\sin \Theta(\mathcal{U}_k, \mathcal{U}_k')\| \leq \frac{\|\Delta\|_2}{\delta_k - \|\Delta\|_2} < \frac{\delta_k/2}{\delta_k - \delta_k/2} = 1    
\end{equation}

\noindent \textbf{Step 3: Bounding $\langle X, \Delta \rangle$.} 
In this step, we aim to show that the inner product $\langle X, \Delta \rangle$ is minimal when $\eta$ is sufficiently small and the injected signals $y_c$ are sufficiently novel to the dominant subspace of $X$. First, from the definition of $\Delta$:
\begin{align}
\langle X, \Delta \rangle = \sum_{c \in \mathcal{C}_{\text{low}}} \langle x_c, \Delta_c \rangle = \sum_{c \in \mathcal{C}_{\text{low}}} (1-\alpha_c) \langle x_c, y_c - x_c \rangle
\end{align}

The absolute value of this expression can be bounded using the triangle inequality and the condition $\alpha_c \in [0,1]$ (which implies $|1-\alpha_c| \leq 1$):
\begin{equation} \label{eq:start}
|\langle X, \Delta \rangle| \leq \sum_{c \in \mathcal{C}_{\text{low}}} |\langle x_c, y_c - x_c \rangle| \leq \sum_{c \in \mathcal{C}_{\text{low}}} \left( |\langle x_c, y_c \rangle| + \|x_c\|_2^2 \right)
\end{equation}
Here, we know that $\|x_c\|_2^2 = I_c \leq \eta$.

Now, we aim to bound $|\langle x_c, y_c \rangle|$. First, let the Singular Value Decomposition (SVD) of $X$ be $X = \sum_{i=1}^r \sigma_i u_i v_i^\top$. Its $c$-th column vector $x_c$ can be expressed as:
\begin{equation}
\label{eq:decomposition}
x_c = \sum_{i=1}^r \sigma_i v_{i,c} u_i
\end{equation}
Here, $v_{i,c}$ is the $c$-th component of the $i$-th right singular vector $v_i$. Then, we can substitute the decomposition~\eqref{eq:decomposition} into the inner product:
\begin{equation} 
\label{eq:14}
\langle x_c, y_c \rangle = \left\langle \sum_{i=1}^r \sigma_i v_{i,c} u_i, y_c \right\rangle = \sum_{i=1}^r \sigma_i v_{i,c} \langle u_i, y_c \rangle
\end{equation}

We now decompose $y_c$ into its projections onto the dominant and residual subspaces:
\begin{equation}
\label{eq:decompose}
   y_c = \underbrace{\sum_{i=1}^k \langle y_c, u_i \rangle u_i}_{y_c^{\text{dominant}}} + \underbrace{\sum_{i=k+1}^r \langle y_c, u_i \rangle u_i}_{y_c^{\text{residual}}} 
\end{equation}

Substituting into Equation~\eqref{eq:14}:
\begin{align}
\langle x_c, y_c \rangle &= \left\langle \sum_{i=1}^r \sigma_i v_{i,c} u_i, y_c^{\text{dominant}} + y_c^{\text{residual}} \right\rangle \\
&= \sum_{i=1}^r \sigma_i v_{i,c} \langle u_i, y_c^{\text{dominant}} \rangle + \sum_{i=1}^r \sigma_i v_{i,c} \langle u_i, y_c^{\text{residual}} \rangle
\end{align}

Taking absolute values:
\begin{equation}
|\langle x_c, y_c \rangle| = \left|\sum_{i=1}^r \sigma_i v_{i,c} \langle u_i, y_c^{\text{dominant}} \rangle + \sum_{i=1}^r \sigma_i v_{i,c} \langle u_i, y_c^{\text{residual}} \rangle\right|
\end{equation}

Applying the triangle inequality:
\begin{equation}
\label{eq:triangle}
|\langle x_c, y_c \rangle| \leq \underbrace{\left|\sum_{i=1}^r \sigma_i v_{i,c} \langle u_i, y_c^{\text{dominant}} \rangle\right|}_{T_1} + \underbrace{\left|\sum_{i=1}^r \sigma_i v_{i,c} \langle u_i, y_c^{\text{residual}} \rangle\right|}_{T_2}
\end{equation}

Since we define $y_c^{\text{dominant}}$ as  $y_c^{\text{dominant}} = \sum_{j=1}^k \langle y_c, u_j \rangle u_j$ in Equation~\eqref{eq:decompose}: 
\begin{align}
T_1 &= \left|\sum_{i=1}^r \sigma_i v_{i,c} \left\langle u_i, \sum_{j=1}^k \langle y_c, u_j \rangle u_j \right\rangle\right| \\
&= \left|\sum_{j=1}^k \langle y_c, u_j \rangle \sum_{i=1}^r \sigma_i v_{i,c} \langle u_i, u_j \rangle\right| \\
&= \left|\sum_{j=1}^k \langle y_c, u_j \rangle (\sigma_j v_{j,c})\right| \quad \text{(by orthonormality of $u_i$)}
\end{align}

Now we can bound $T_1$ by applying the Cauchy-Schwarz inequality:
\begin{align}
T_1 &\leq \left(\sum_{j=1}^k |\langle y_c, u_j \rangle|^2\right)^{1/2} \left(\sum_{j=1}^k |\sigma_j v_{j,c}|^2\right)^{1/2} \\
\label{eq:t1-bound}
&\leq \left(\sum_{j=1}^k |\langle y_c, u_j \rangle|^2\right)^{1/2} \sqrt{I_c} \quad \text{(since $I_c = \sum_{j=1}^r \sigma_j^2 v_{j,c}^2$)}
\end{align}

Now let us bound $T_2$. Since we define $y_c^{\text{residual}}$ in Equation~\eqref{eq:decompose} as $\langle u_i, y_c^{\text{residual}} \rangle = 0$ for $i \leq k$:
\begin{align}
T_2 &= \left|\sum_{i=1}^r \sigma_i v_{i,c} \langle u_i, y_c^{\text{residual}} \rangle\right| \\
&= \left|\sum_{i=k+1}^r \sigma_i v_{i,c} \langle u_i, y_c^{\text{residual}} \rangle\right|
\end{align}

Applying the Cauchy-Schwarz inequality two times:
\begin{align}
T_2 &\leq \left(\sum_{i=k+1}^r |\sigma_i v_{i,c}|^2\right)^{1/2} \left(\sum_{i=k+1}^r |\langle u_i, y_c^{\text{residual}} \rangle|^2\right)^{1/2} \\
\label{eq:t2-bound}
&\leq \sqrt{I_c} \cdot \left(\sum_{i=k+1}^r |\langle u_i, y_c^{\text{residual}} \rangle|^2\right)^{1/2}
\end{align}

Since $y_c^{\text{residual}} = \sum_{i=k+1}^r \langle y_c, u_i \rangle u_i$:
\begin{equation}
\label{eq:t2-bound-final}
~\eqref{eq:t2-bound} = \sqrt{I_c} \cdot \|y_c^{\text{residual}}\|_2
\end{equation}

Now we combine both bounds from Equation~\eqref{eq:t1-bound} and Equation~\eqref{eq:t2-bound-final}. Recalling from Equation~\eqref{eq:triangle}:
\begin{align}
\label{eq:recalling-triangle}
|\langle x_c, y_c \rangle| &\leq T_1 + T_2
\end{align}

By applying \textit{Assumption 4}, where $\sum_{i=1}^k |\langle y_c, u_i \rangle|^2 \leq \gamma \sum_{i=k+1}^r |\langle y_c, u_i \rangle|^2 = \gamma \|y_c^{\text{residual}}\|_2^2$ to both Equation~\eqref{eq:t1-bound} and Equation~\eqref{eq:t2-bound-final}:
\begin{align}
 T_1 + T_2 &\leq \sqrt{\gamma} \|y_c^{\text{residual}}\|_2  \sqrt{I_c} + \sqrt{I_c} \|y_c^{\text{residual}}\|_2 \\
\label{eq:final-bound}
&= \sqrt{I_c} \|y_c^{\text{residual}}\|_2 (\sqrt{\gamma} + 1) 
\end{align}

Since $\|y_c^{\text{residual}}\|_2 \leq \|y_c\|_2 \leq \beta$ (by \textit{Assumption 1}) and $I_c \leq \eta$ for $c \in \mathcal{C}_{\text{low}}$, by combining Equation~\eqref{eq:recalling-triangle} and Equation~\eqref{eq:final-bound}, we finally get:
\begin{equation}
|\langle x_c, y_c \rangle| \leq \beta \sqrt{\eta} (\sqrt{\gamma} + 1)
\end{equation}

Now, we substitute this bound and $\|x_c\|_2^2 = I_c \leq \eta$ into Equation \eqref{eq:start}:
\begin{align}
|\langle X, \Delta \rangle| &\leq \sum_{c \in \mathcal{C}_{\text{low}}} \left( |\langle x_c, y_c \rangle| + \|x_c\|_2^2 \right) \\
&\leq \sum_{c \in \mathcal{C}_{\text{low}}} \left( \beta \sqrt{\eta} (\sqrt{\gamma} + 1) + \eta \right) \\
\label{eq:step3-result}
&= |\mathcal{C}_{\text{low}}| \left( \beta (\sqrt{\gamma} + 1) \sqrt{\eta} + \eta \right)
\end{align}

This final bound shows that the inner product between $X$ and $\Delta$ becomes very small when $\eta$ is sufficiently small and $\gamma$ is bounded. The bounded alignment assumption (small $\gamma$, \textit{Assumption 4}) ensures that the injected signals $y_c$ contribute primarily to the residual subspace rather than interfering with the dominant patterns of $X$.

\noindent \textbf{Step 4: Spectral flattening.} In this step, we combine the previous results to analyze the change in the representation's overall magnitude, measured by the squared Frobenius norm $\|X'\|_F^2$. 
The squared Frobenius norm is defined as the inner product of a matrix with itself as below.
\begin{equation}
\|X'\|_F^2 = \langle X', X' \rangle   
\end{equation}

Since $X' = X + \Delta$:
\begin{equation}
    \|X'\|_F^2 = \langle X + \Delta, X + \Delta \rangle
\end{equation}

Using the distributive property (bilinearity) of the inner product:
\begin{equation}
    \langle X + \Delta, X + \Delta \rangle = \langle X, X \rangle + \langle X, \Delta \rangle + \langle \Delta, X \rangle + \langle \Delta, \Delta \rangle
\end{equation}

Since $\langle X, \Delta \rangle = \langle \Delta, X \rangle$ for the Frobenius inner product, the expression simplifies to:
\begin{equation}
    \|X'\|_F^2 = \|X\|_F^2 + 2\langle X, \Delta \rangle + \|\Delta\|_F^2
\end{equation}

Rearranging the equation,
\begin{equation}
    \|X'\|_F^2 - \|X\|_F^2 = 2\langle X, \Delta \rangle + \|\Delta\|_F^2
\end{equation}

Combining the result of \textbf{Step 3} (Equation~\eqref{eq:step3-result}) and \textit{Assumption 2} ($\|\Delta\|_F^2 \geq \epsilon$):
\begin{align}
\|X'\|_F^2 - \|X\|_F^2 &= 2\langle X, \Delta \rangle + \|\Delta\|_F^2 \\
&\geq -2|\mathcal{C}_{\text{low}}| \left( \beta (\sqrt{\gamma} + 1) \sqrt{\eta} + \eta \right) + \epsilon \\
&> \epsilon/2 > 0
\end{align}
for sufficiently small $\eta$ and $\gamma$. This positive change is injected primarily into the tail subspace, as ensured by the bounded alignment assumption, leading to a flatter singular value spectrum.

\noindent \textbf{Step 5: Effective rank increase.}
From \textbf{Step 2}, the dominant subspace remains stable. Also, from \textbf{Step 4}, the squared Frobenius norm of the representation increases, i.e., $\|X'\|_F^2 > \|X\|_F^2 + \epsilon/2$. The bounded alignment condition (\textit{Assumption 4}) ensures that the modification $\Delta$ is poorly aligned with the dominant subspace of $X$. This, combined with the fact that the update is restricted to low-informativeness channels (which are inherently less aligned with the dominant subspace), suggests that the increase in squared Frobenius norm is primarily concentrated in the residual singular values. 

Therefore, we have:
\begin{align}
\sum_{j\leq k} \sigma_j'^2 &\approx \sum_{j\leq k} \sigma_j^2 \quad \text{(dominant spectrum preserved)} \\
\sum_{j>k} \sigma_j'^2 &> \sum_{j>k} \sigma_j^2 + \epsilon/2 \quad \text{(residual spectrum amplified)}
\end{align}

Let $p$ and $p'$ be the normalized singular value spectra of $X$ and $X'$. This redistribution of squared singular values, where the total sum increases while the dominant part remains roughly constant, leads to a \textit{flattening} of the normalized singular value distribution $p'$ compared to $p$. Specifically, the increase in total mass ($\|X'\|_F^2 > \|X\|_F^2$) is concentrated in the tail, which reduces the probability mass of the dominant components. Since the Shannon entropy $H(p)$ is Schur-concave, it increases when the probability mass is transferred from larger components to smaller ones~\citep{marshall1979inequalities}. Therefore:
\begin{equation}
    H(p') > H(p) \implies \mathrm{ERank}(X') > \mathrm{ERank}(X)
\end{equation}

\end{proof}


\section{Analysis}
\subsection{Additional Analysis on Spectrum Flattening}
\label{sup:fig2-explained}
\input{table/orthogonality}
In Figure~\ref{fig:eigen-spectra}, we show that the extent of improvement varies across datasets. For instance, in UTKinect (Figure~\ref{fig:eigen-spectra} (b)), Depth shows more pronounced spectral expansion than RGB, whereas in DARai (Figure~\ref{fig:eigen-spectra}  (a),(c)) and NTURGBD (Figure~\ref{fig:eigen-spectra} (d)), RGB shows more expansion than Depth. To explain this, Table~\ref{tab:orthogonality} quantifies the complementarity between modalities via principal angles (Assumption 4 in Theorem~\ref{thm:main}). In UTKinect, the angle from RGB to Depth (65.57°) is larger than the reverse (63.80°), suggesting that RGB contributes more novel information to Depth, consistent with the stronger spectral expansion observed in Depth shown in Figure~\ref{fig:eigen-spectra} (b). In contrast, the smaller and more symmetric angles in DARai and NTURGBD correspond to more prominent gains in RGB. These findings validate our theoretical insight that the degree of complementarity between modalities affects how much each can benefit from fusion.

\section{Architecture}
\subsection{Temporal Fuser Architecture}
\label{lab:temporal-fuser}
\vspace{0.2cm}
\noindent \textbf{MHSA.} To understand how different frames relate to each other and provide contextual meaning to each scene, we employ a MHSA within each temporal fusion layer. Unlike traditional sequential models (e.g., RNNs), self-attention enables the model to dynamically attend to the most relevant frames, ensuring that each time step is interpreted in the correct context. Given an input of the stacked multi-modal feature sequence $\mathbf{F}_{\text{stacked}} = \begin{bmatrix} \mathbf{F}^{\text{RGB}} \\ \mathbf{F}^{\text{Depth}} \end{bmatrix} \in \mathbb{R}^{B \times T \times 2C}$, where \( B \) is the batch size, \( T \) is the video sequence length, and \( 2C \) represents the concatenated feature dimensions from both modalities, we define the MHSA operation as:
\begin{equation}
    \text{MHSA}(\mathbf{F}^{(l)}) = \text{Concat}(\mathbf{A}_1, \mathbf{A}_2, \dots, \mathbf{A}_h) \mathbf{W}_o,
\end{equation}
where $\mathbf{F}^{(l)}$ is the input sequence representation in $l^{\text{th}}$ layer, \( \text{Concat}(\cdot) \) denotes a concatenation function, \( h \) is the number of attention heads and \( \mathbf{W}_o \in \mathbb{R}^{(h \cdot d_h) \times d} \) is the output projection matrix.

Each attention head computes:
\begin{equation}
    \mathbf{A}_i = \text{Softmax}\left(\frac{\mathbf{Q}_i \mathbf{K}_i^\top}{\sqrt{d_h}}\right) \mathbf{V}_i,
\end{equation}
where \( \mathbf{Q}_i = \mathbf{F}^{(l)} \mathbf{W}_i^Q \), \( \mathbf{K}_i = \mathbf{F}^{(l)} \mathbf{W}_i^K \), \( \mathbf{V}_i = \mathbf{F}^{(l)} \mathbf{W}_i^V \), are the query, key, and value projection matrices, respectively, \( \mathbf{W}_i^Q, \mathbf{W}_i^K, \mathbf{W}_i^V \in \mathbb{R}^{d \times d_h} \), and \( d_h = \frac{d}{h} \) is the feature dimension per head.

To ensure stable training, we apply Layer Normalization \( \text{LN}(\cdot) \) and residual connections:
\begin{equation}
    Z_t = \text{LN}(F_{\text{stacked}}) + \text{MHSA}(\text{LN}(F_{\text{stacked}})).
\end{equation}
\vspace{0.2cm}
\noindent \textbf{MLP.} While MHSA models dependencies across time, it does not directly enhance the expressiveness of individual frame features. To refine the representation of each frame, we apply an \( \text{MLP}(\cdot) \) to introduce non-linearity and feature transformation. The MLP module operates as:
\begin{align}
    \mathbf{F}^{(l+1)} &= \text{LN}(\mathbf{Z}^{(l)} + \text{MLP}(\mathbf{Z}^{(l)})),
\end{align}
where \( \text{LN}(\cdot) \) and \( \text{MLP}(\cdot) \) are applied independently at each time step.

After \( L \) layers of temporal fusion, we obtain the final output representation:

\begin{equation}
    \mathbf{F}_{\text{final}} = \text{LN}(\mathbf{F}^{(L)}) \in \mathbb{R}^{B \times T \times d}.
\end{equation}

\vspace{0.2cm}
\noindent \textbf{Segmentation head.} The ultimate goal of the Temporal Fuser is to determine what each frame represents by assigning an action label to every time step. To accomplish this, we employ a segmentation head, which consists of a fully connected layer $\text{FC}(\cdot)$: \begin{equation} \mathbf{\hat{\mathbf{Y}}} = \text{FC}(\mathbf{F}_{\text{final}}),\end{equation}
where $\hat{y_{t}}$ the predicted action class at time step $t$.

\subsection{Action anticipation Module Architecture}
\label{lab:action-anticipation-module-arc}
\vspace{0.2cm}
\noindent \textbf{Future queries as learnable representations.} Instead of relying on a single deterministic output for action anticipation, we introduce learnable queries—denoted as future queries ($\mathbf{Q}_{\text{future}}$) to predict multiple action possibilities \cite{gong2022future}. These queries are randomly initialized and learned during training. Given $\mathbf{N_q}$ future queries, we define: $\mathbf{Q}_{\text{future}} \in \mathbb{R}^{N_q \times C}$ where $C$ is the hidden dimension. These queries are independent of any specific frame and instead act as a learnable representation that extracts relevant information from the past observations.

\vspace{0.2cm}
\noindent \textbf{Multi-head cross-attention (MHCA).} To effectively anticipate future actions, the model must attend to relevant moments in the past. To achieve this, we employ MHCA, where the future queries ($\mathbf{Q}_{\text{future}}$) interact with the Temporal Fuser's output ($\mathbf{F}_{\text{temporal}}$).
\begin{equation}
\mathbf{H}_{\text{attn}}^{(h)} = \text{Attention}(\mathbf{Q}_{\text{future}} \mathbf{W}_Q^{(h)}, \mathbf{F}_{\text{temporal}} \mathbf{W}_K^{(h)}, \mathbf{F}_{\text{temporal}} \mathbf{W}_V^{(h)}),
\end{equation} where \(h\) denotes the head index, and \(\mathbf{W}_Q^{(h)}, \mathbf{W}_K^{(h)}, \mathbf{W}_V^{(h)}\) are learnable projection matrices for queries, keys, and values, respectively.

The outputs from all attention heads are concatenated and projected through a learnable weight matrix \(\mathbf{W}_O\):
\begin{equation}
    \mathbf{H}_{\text{MHCA}} = \text{Concat}(\mathbf{H}_{\text{attn}}^{(1)}, \mathbf{H}_{\text{attn}}^{(2)}, \dots, \mathbf{H}_{\text{attn}}^{(H)}) \mathbf{W}_O,
\end{equation}
where $H$ denotes the number of attention head and  \(\mathbf{W}_O\) ensures that the output dimensionality remains consistent.

This formulation allows the model to refine future action representations by selectively attending to relevant past temporal features.

\vspace{0.2cm}
\noindent \textbf{Feed-forward network (FFN).} The output of $\text{MHCA}(\cdot)$ is further processed through a $\text{FFN}(\cdot)$ to capture non-linear dependencies in the learned future queries:
$\text{FFN}(\mathbf{X}) = \sigma(\mathbf{X} \mathbf{W}_1 + \mathbf{b}_1) \mathbf{W}_2 + \mathbf{b}_2,$
where $\mathbf{W}_1 \in \mathbb{R}^{d \times d_f}$, $\mathbf{W}_2 \in \mathbb{R}^{d_f \times d}$, and $\sigma(\cdot)$ is a non-linear activation.

\vspace{0.2cm}
\noindent \textbf{Action anticipation head.} The refined representation from the $\text{FFN}(\cdot)$ is used to anticipate the future action label using a $\text{FC}(\cdot)$ layer: \begin{equation} \hat{y}_{\text{future}} = \text{FC}(\text{FFN}(\mathbf{X})),
\end{equation}
where $\hat{y}_{\text{future}}$ denotes the predicted future action class.

\section{Experimental Details}
\subsection{Datasets}
\label{lab:datasets}
We aim to evaluate \texttt{R3D}'s ability to anticipate human actions in diverse and realistic scenarios, particularly in settings where spatial reasoning and multi-modal fusion play a critical role. To this end, we utilize three action anticipation datasets: NTURGBD \cite{shahroudy2016ntu}, UTKinect-Action3D \cite{xia2012view}, and DARai \cite{kaviani2025hierarchical}.
The UTKinect-Action3D dataset was collected by the University of Texas at Austin and features 10 types of human actions, each performed twice by 10 subjects from multiple viewpoints. The dataset poses additional challenges due to actor-dependent variations, occlusions caused by human-object interactions, and body parts moving out of the field of view.
The NTURGBD dataset is one of the largest and most diverse action recognition datasets, designed for 3D human action analysis. It contains 56,880 video samples covering 60 action classes performed by 40 subjects from 80 viewpoints.
The DARai dataset offers highly realistic scenarios that closely resemble real-world human behavior. Unlike UTKinect and NTURGBD, the videos in DARai are untrimmed, showcasing raw, continuous human activity in real-world contexts without artificial segmentation. It comprises 150 action classes, 50 participants, two distinct exocentric views, and three levels of hierarchical labels.

\subsection{Experimental setups}
\label{lab:setup}
We use pretrained ResNet features as input visual features for the NTURGBD, UTKinects, and DARai. To align with the temporal resolution of each dataset, the sampling rates are set to 15 for DARai, and 1 for UTKinects and NTURGBD. We use the number of Future Queries $\mathbf{N_{\text{q}}}$ fixed at 8. Based on the density of each dataset, the hidden dimension size $D$ is set to 128 for both NTURGBD, UTKinects, and DARai. During training, the observation rate $\alpha$ is set to $\alpha \in \{0.2, 0.3, 0.5\}$, while the prediction rate $\beta$ is fixed at 0.5. Yet when training UTKinects, to further augment the training dataset, we use the observation rate $\alpha \in \{0.2, 0.25, 0.3, 0.35, 0.4, 0.45, 0.5, 0.55\}$.  The model is trained for 60 epochs using the AdamW optimizer \cite{loshchilov2017decoupled} with a learning rate of $1e-3$ and a batch size of 8. A cosine annealing warm-up scheduler \cite{loshchilov2016sgdr} is applied, with warm-up stages spanning 10 epochs. For evaluation setup, we follow the long-term action anticipation framework protocol~\citep{abu2018will,abu2021long,sener2020temporal,ke2019time}. We set the observation rate $\alpha \in \{0.2, 0.3\}$ and prediction rate $\beta \in \{0.1, 0.2, 0.3, 0.5\}$, and measure mean over classes (MoC) accuracy for evaluation metrics. To ensure consistency, we report average performance across 3 number of iteration, each with fixed seeds $1, 10, 13452$.

\subsection{Computational cost and Scalabilty}
\label{lab:comp-cost-scalability}
\input{table/computation-cost}

\noindent \textbf{Computational cost.} We evaluate the computational cost and scalability of \texttt{R3D} on an NVIDIA A40 GPU. For comparison, we chose GTAN~\citep{zatsarynna2024gated}, the current state-of-the-art model reported in Table~\ref{tab:main-experiment}. As shown in Table~\ref{tab:computational-cost}, \texttt{R3D} processes a single RGB-Depth frame in 0.119 ms, requiring 61.77MB of memory and 0.58 GFLOPs. In contrast, GTAN requires 5.92ms per RGB frame and approximately 49 GFLOPs. This gap arises because GTAN is a diffusion-based model whose computational cost scales linearly with the number of sampling steps, while \texttt{R3D} avoids such overhead.

\noindent \textbf{Scalability.} As shown in Table~\ref{tab:computational-cost} without \texttt{RTF}, inference on a single RGB-Depth frame costs 0.02ms per frame. This is because \texttt{RTF} involves SVD operations, which is the main contributor to runtime. To mitigate this overhead, we experiment with lowering the channel dimension from 128 to 64 before the \texttt{RTF} stage on the UTKinects dataset. This adjustment reduced the per-frame cost from 0.119ms to 0.09ms and from 0.58 GFLOPs to 0.24 GFLOPs, while still achieving state-of-the-art performance. In fact, the results reported in Table~\ref{tab:main-experiment} for UTKinects are based on this scaled configuration, showing that hyperparameter tuning enables us to maintain accuracy while significantly improving computational efficiency.

\section{More Ablation Studies}
\input{table/ablation-temporal}
\input{table/ablation-alignment}

\subsection{Impact of Temporal Fuser}
\label{sup:ablation-study}
We conduct an ablation study on the Temporal Fuser, which models temporal dependencies across RGB and depth features. As observed in Table \ref{tab:ablation-temporal}, removing the Temporal Fuser leads to a noticeable performance drop, highlighting its role in leveraging temporal context for action anticipation.

\subsection{Ablation on Latent Space Projection}
\label{sup:ablation-latent}
To evaluate the impact of aligning modalities in a shared feature space prior to fusion, we conduct an ablation study comparing two settings: Shared, where modality-specific features are projected into a common latent space before entering the Rank-enhancing Token Fuser, and Separated, where fusion is performed without this projection step. As shown in Table~\ref{tab:alignment-ablation}, the Separated setting consistently  outperforms the Shared one. We hypothesize that this performance gap arises because projecting both modalities into the same latent space reduces the structural differences between them. This highlights the importance of maintaining modality-specific representations before fusion in order to fully leverage richer and more complementary information from both modalities.

\begin{figure}[t!]
\begin{center}
\includegraphics[width=1\linewidth]{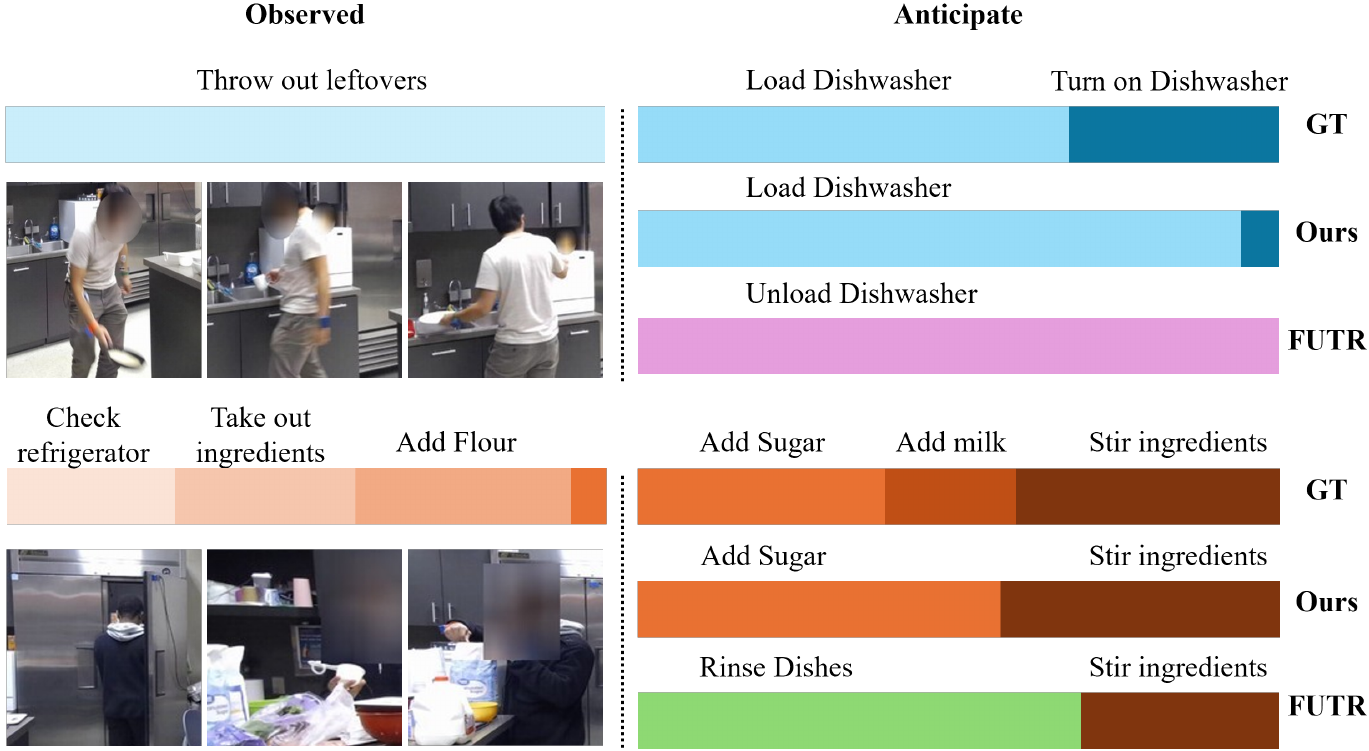}
\end{center}
\vspace{-0.5cm}
   \caption{Qualitative comparison of action anticipation results between the Ground Truth (GT), the RGB-based baseline model (GTAN), and \texttt{R3D}.}
\label{fig:qualitative}
\end{figure}

\section{Qualitative Analysis}
\label{lab:qualitative}
To assess the ability to distinguish fine-grained human actions, we compare \texttt{R3D} against the strongest baseline from Table \ref{tab:main-experiment}, using the most detailed label set—L3 (procedure) level from the DARai dataset. Figure \ref{fig:qualitative} presents a qualitative comparison between the ground truth (GT), the RGB-based GTAN model, and \texttt{R3D}. In the first example of Figure \ref{fig:qualitative}, the task requires differentiating between loading and unloading a dish into a dishwasher. The RGB-only GTAN model struggles with this distinction as it lacks explicit spatial cues to determine whether the dish is moving toward or away from the dishwasher. This results in misclassification due to the absence of motion directionality cues. In contrast, by allowing the model to infer that the dish is moving toward the dishwasher, \texttt{R3D} correctly predicts `loading' by leveraging depth-based directionality. In the second example of Figure \ref{fig:qualitative}, the person first adds flour, then quickly transitions to adding sugar. Due to the brevity of this transition, only a small portion of the add sugar action is observed before the model must anticipate the next step. \texttt{R3D} successfully distinguishes between adding sugar and adding flour, demonstrating its ability to capture subtle procedural differences. This is achieved as \texttt{R3D} effectively utilizes the strengths of RGB features for recognizing sugar and flour.

\begin{figure}[t!]
\begin{center}
\includegraphics[width=1\linewidth]{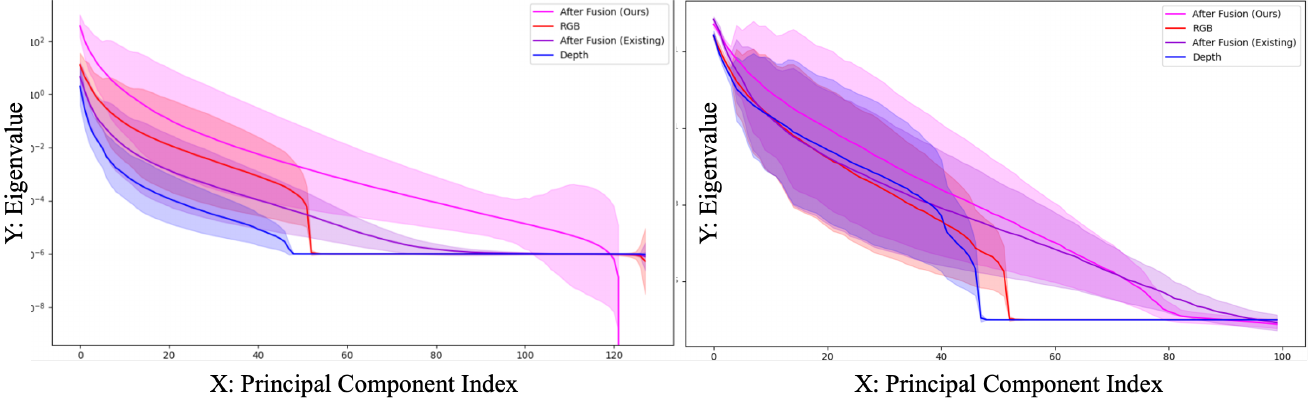}
\end{center}
\vspace{-0.5cm}
   \caption{This figure compares the eigenvalue distribution of RGB (red), Depth (blue), existing fusion method (purple), and our proposed method (pink) in DARai (Left) and NTURGBD (Right) Dataset. Unlike the existing fusion method, where the eigenvalues fall below those of the RGB modality, our method exhibits a broader and higher eigenvalue spectrum than both RGB and Depth.}
\label{fig:modality-collapse}
\vspace{-0.5cm}
\end{figure}

\begin{figure}[t!]
\begin{center}
\includegraphics[width=1\linewidth]{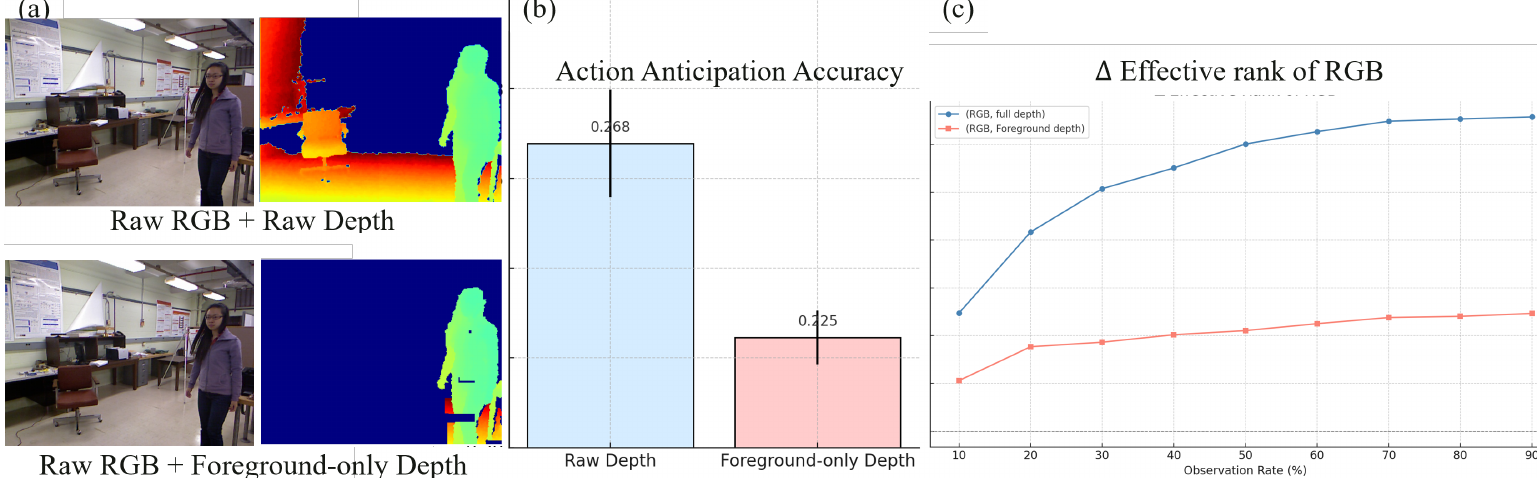}
\end{center}
\vspace{-0.5cm}
   \caption{Analysis of the impact of background depth in RGB-depth fusion for action anticipation. (a) Two types of inputs from the UTKinect dataset: RGB images paired with either raw depth (top) or foreground-only depth (bottom), where foreground regions are segmented using joint annotations. (b) Average action anticipation performance (across 3 seeds) when fusing RGB with raw depth (blue) versus foreground-only depth (red). (c) Change in the effective rank of RGB representations as a function of observation rate, highlighting that raw depth significantly enhances the expressive capacity of RGB features compared to foreground-only depth.}
\label{fig:foreground-depth}
\vspace{-0.2cm}
\end{figure}

\begin{figure}[t!]
\begin{center}
\includegraphics[width=1\linewidth]{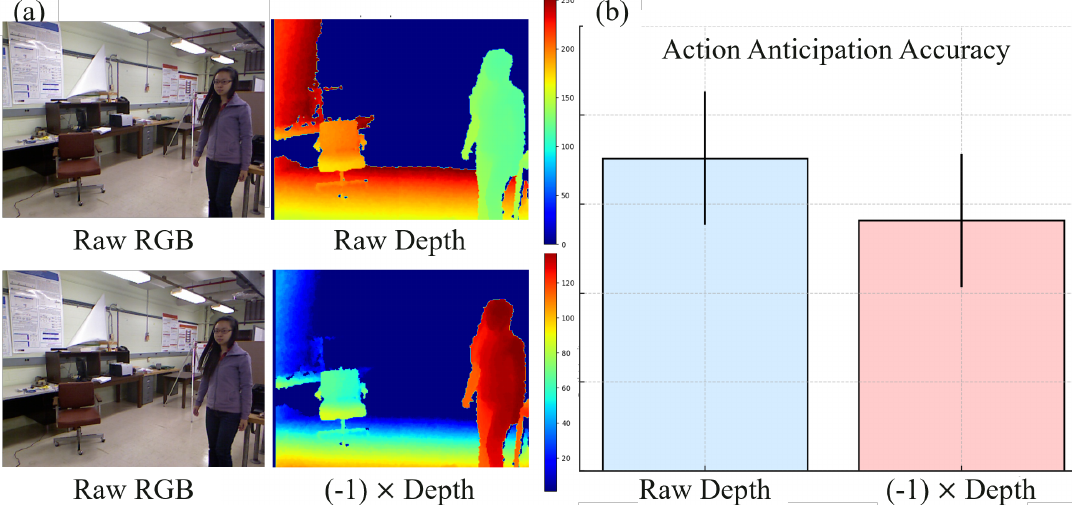}
\end{center}
\vspace{-0.5cm}
   \caption{Analysis of the impact of background depth in RGB-depth fusion for action anticipation. (a) shows two types of inputs from the UTKinect dataset: RGB images paired with either raw depth (top) or foreground-only depth (bottom), where foreground regions are segmented using joint annotations. (b) Average action anticipation performance (across 3 seeds) when fusing RGB with raw depth (blue) versus foreground-only depth (red).}
\label{fig:directionality-analysis}
\vspace{-0.5cm}
\end{figure}

\section{Discussion: why Depth?}
To support our choice of depth as the complementary modality in Section~\ref{sec:modality-selection}, we conduct a qualitative analysis to better understand why depth is particularly effective for multimodal fusion.

\subsection{Depth Mitigates Modality Collapse} We investigate whether using depth indeed mitigates modality collapse, as hypothesized in our modality selection strategy. In Table~\ref{tab:main-experiment}, the largest performance gap between our model and the SOTA baseline occurs on the DARai dataset, whereas the smallest gap is observed on the NTURGB+D dataset. We analyze both cases by plotting the eigenvalue spectra of the fused features. As shown in Figure~\ref{fig:modality-collapse}, We observe that the existing method (AFFT) exhibits modality collapse: in the DARai dataset (left), its fusion spectrum lies below that of RGB, indicating a collapse of RGB features; in contrast, on the NTURGB+D dataset (right), the fusion spectrum falls below that of depth. In both cases, our method maintains a broader and higher-magnitude spectrum. Notably, this trend holds on the both datasets, even when the performance gap between models is relatively small, reinforcing the effectiveness of our fusion strategy in preserving modality diversity.

\subsection{Depth Delivers Contextual Information} To demonstrate the importance of depth information, as shown in Figure~\ref{fig:foreground-depth}, we conduct experiments on the UTKinect dataset, where most actions occur primarily in the foreground. Using the provided joint annotations, we segment the depth maps into foreground-only depth and raw depth as shown in Figure~\ref{fig:foreground-depth} (a). We then fuse each version with RGB and compare their performance on the action anticipation task. Figure~\ref{fig:foreground-depth} (b) shows that even in a foreground-centric dataset where foreground-only information might appear sufficient, raw depth consistently yields better performance. To better understand this phenomenon, we measure the change in Effective Rank of the RGB stream. As shown in Figure~\ref{fig:foreground-depth} (c), when raw depth which includes background is used, RGB representations exhibit a more substantial increase in expressive capacity. In contrast, with foreground-only depth, the increase in RGB representation is notably smaller. In conclusion, depth provides more than just motion cues from the foreground—it also delivers critical spatial and contextual information from the background. Thus, even background depth plays a vital role in structuring scene understanding, reinforcing depth as a key modality.

\subsection{Directionality in Depth Features Matters} As shown in Figure~\ref{fig:directionality-analysis}, to test the importance of directionality in depth features, we negate the depth features (i.e., multiply by $-1$) before fusion as shown in Figure~\ref{fig:directionality-analysis} (a). This operation preserves the magnitude of feature variation but flips their geometric interpretation. The performance drop in Figure~\ref{fig:directionality-analysis} (b) suggests that directionality of depth information matters in multi-modal fusion. This observation motivates the need for modality-aware and direction-sensitive fusion.

%% file: table/orthogonality.tex

\begin{table}[ht]
\centering
\scalebox{0.7}{
\begin{tabular}{c c c}
\toprule
\textbf{Dataset} & \(\mathbf{X \leftarrow Y}\) (Fusion Direction) & \textbf{Angle between subspaces (°)} \\
\midrule
\multirow{2}{*}{\shortstack{DARai\\(Coarse)}} & RGB   \(\leftarrow\) Depth  & 48.74\\
                                              & Depth \(\leftarrow\) RGB    & 47.73\\
\midrule
\multirow{2}{*}{UTKinect} & RGB   \(\leftarrow\) Depth  & 63.80\\
                                              & Depth \(\leftarrow\) RGB    & 65.57\\
\midrule
\multirow{2}{*}{\shortstack{DARai\\(Fine-grained)}} & RGB   \(\leftarrow\) Depth  & 48.68\\
                                                   & Depth \(\leftarrow\) RGB    & 47.20\\
\midrule
\multirow{2}{*}{NTURGBD}                      & RGB   \(\leftarrow\) Depth  & 34.25\\
                                              & Depth \(\leftarrow\) RGB    & 31.95\\
\bottomrule
\end{tabular}}
\caption{This shows an complementarity (in degrees) between the principal subspaces \(\mathbf{X}\) and the injected channels \(\mathbf{Y_{C_{low}}}\). Higher values (closer to 90°) indicate stronger complementarity, which contributes more effectively to rank increase.}
\label{tab:orthogonality}
\vspace{-0.3cm}
\end{table}

%% file: table/computation-cost.tex
\begin{table}[t]
\centering
\scalebox{0.70}{
\begin{tabular}{lccc}
\toprule
Method & Time / frame (ms) & FLOPs (per frame) \\
\midrule
GTAN~\citep{zatsarynna2024gated} & 5.920 & 49 GFLOPs \\
\midrule
\textbf{Ours} & 0.119 & 0.58 GFLOPs \\
\quad without \texttt{RTF}    & 0.020 & 0.12GFLOPs    \\
\quad 128 $\rightarrow$ 64 channels & 0.090 & 0.24 GFLOPs   \\
\bottomrule
\end{tabular}}
\caption{Comparison of computational cost and scalability on an NVIDIA A40 GPU.}
\label{tab:computational-cost}
\end{table}

%% file: table/ablation-temporal.tex
\begin{table}[t!]
\centering
\renewcommand{\arraystretch}{1.2}
\scalebox{0.5}{
\begin{tabular}{@{}llcccccccccccc@{}}
\toprule
\multirow{3}{*}{Dataset} & \multirow{3}{*}{Methods} & \multicolumn{4}{c}{$\beta (\alpha=0.2)$} 
& \multicolumn{4}{c}{$\beta (\alpha=0.3)$} \\  
\cmidrule(lr){3-6} \cmidrule(lr){7-10}
 &  & 0.1 & 0.2 & 0.3 & 0.5 & 0.1 & 0.2 & 0.3 & 0.5 \\  
\midrule
\multirow{2}{*}{UTKinects}  
 & W/O & 31.31 & 23.02 & 22.11 & 17.18 & 13.40 & 17.31 & 15.78 & 11.34 \\  
 & W/ & \textbf{38.96} & \textbf{37.68} & \textbf{37.16} & \textbf{35.95} &
 \textbf{34.80} &
 \textbf{32.51} &
 \textbf{36.03} &
 \textbf{26.98}\\  
\midrule
\multirow{2}{*}{NTURGBD}  
 & W/O & 16.77 & 16.59 & 16.59 & 16.62 & 20.64 & 20.73 & 20.76 & 20.85  \\  
 & W/ & \textbf{21.98} & \textbf{20.18} & \textbf{19.72} & \textbf{18.90} & \textbf{23.17} & \textbf{21.27} & \textbf{20.51} & \textbf{19.89} \\ 
\midrule
\multirow{2}{*}{DARai}  
 & W/O & 20.92 & 20.43 & 19.37 & 18.87 & 37.88 & 31.04 & 30.51 & 29.20\\  
 & W/ & \textbf{33.44} & \textbf{32.14} & \textbf{30.56} & \textbf{29.59} & \textbf{46.29} & \textbf{42.05} & \textbf{43.41} & \textbf{40.25} \\  
\bottomrule
\end{tabular}}
\caption{Ablation Study across three datasets evaluating the impact of Temporal Fuser. ``W'' denotes the use of Temporal Fuser, while ``W/O'' indicates its absence.}
\label{tab:ablation-temporal}
\vspace{-0.3cm}
\end{table}

%% file: table/ablation-alignment.tex
\begin{table}[t!]
\centering
\renewcommand{\arraystretch}{1.2}
\scalebox{0.5}{
\begin{tabular}{@{}llcccccccc@{}}
\toprule
\multirow{2}{*}{Dataset} & \multirow{2}{*}{Methods} & \multicolumn{4}{c}{$\beta (\alpha=0.2)$} & \multicolumn{4}{c}{$\beta (\alpha=0.3)$} \\  
\cmidrule(lr){3-6} \cmidrule(lr){7-10} 
 &  & 0.1 & 0.2 & 0.3 & 0.5 & 0.1 & 0.2 & 0.3 & 0.5 \\  
\midrule
\multirow{2}{*}{UTKinects}  
 & Shared & 34.07 & \textbf{40.43} & 35.06 & 27.60 & 27.11 & 26.82 & 34.62 & 25.30 \\  
 & Separated & \textbf{38.96} & 37.68 & \textbf{37.16} & \textbf{35.95} & \textbf{34.80} & \textbf{32.51} & \textbf{36.03} & \textbf{26.98} \\  
\midrule
\multirow{2}{*}{NTURGBD}  
 & Shared & 10.39 & 10.21 & 10.23 & 10.23 & 9.09 & 9.28 & 9.26 & 9.32 \\  
 & Separated & \textbf{21.98} & \textbf{20.18} & \textbf{19.72} & \textbf{18.90} & \textbf{23.17} & \textbf{21.27} & \textbf{20.51} & \textbf{19.89} \\ 
\midrule
\multirow{2}{*}{DARai (Coarse)}  
 & Shared & 27.83 & 28.16 & 26.25 & 24.75 & 38.90 & 31.18 & 32.16 & 30.66 \\  
 & Separated & \textbf{33.44} & \textbf{32.14} & \textbf{30.56} & \textbf{29.59} & \textbf{46.29} & \textbf{42.05} & \textbf{43.41} & \textbf{40.25} \\  
 \midrule
\multirow{2}{*}{DARai (Fine)}  
 & Shared & 29.60 & 25.44 & 23.70 & 16.13 & \textbf{27.46} & 26.12 & 22.17 & 17.65 \\  
 & Separated & \textbf{32.57} & \textbf{25.92} & \textbf{24.02} & \textbf{16.68} & 26.97 & \textbf{28.43} & \textbf{24.81} & \textbf{18.02} \\   
\bottomrule
\end{tabular}}
\caption{Ablation study on the effect of projecting each modality into a shared latent space prior to the Token Fuser. ``Shared" denotes the setting where modality-specific features are projected into a common latent space before fusion, while ``Separated" refers to the setting where modalities are fused without this projection step.}
\label{tab:alignment-ablation}
\vspace{-0.5cm}
\end{table}